%% file: main.tex
\documentclass[11pt]{article}
\usepackage[bookmarks,colorlinks]{hyperref}
\usepackage{xcolor}
\definecolor{darkblue}{rgb}{0,0.22,0.66}
\definecolor{darkcyan}{RGB}{0, 139, 139}
\definecolor{darkgray}{HTML}{666666}
\hypersetup{colorlinks=true,
citecolor=darkcyan,
linkcolor=darkblue,
urlcolor=darkgray
}

\input{preamble}

\title{Multigrade Neural Network Approximation}

\author{\name   Shijun Zhang\thanks{Corresponding author
}
\email  \href{mailto:shijun.zhang@polyu.edu.hk}{shijun.zhang@polyu.edu.hk}\\
\addr  Department of Applied Mathematics\\
Hong Kong Polytechnic University
\AND
\name  Zuowei Shen
\email 
\href{mailto:matzuows@nus.edu.sg}{matzuows@nus.edu.sg}\\
\addr  Department of Mathematics\\
National University of Singapore
\AND  \name  Yuesheng Xu
\email \href{mailto:y1xu@odu.edu}{y1xu@odu.edu}
\\ 
\addr  Department of Mathematics and Statistics\\ Old Dominion University
    }

\begin{document}
\maketitle

\begin{abstract}

We study multigrade deep learning (MGDL) as a principled framework for structured error refinement in deep neural networks. While the approximation power of neural networks is now relatively well understood, training very deep architectures remains challenging due to highly nonconvex and often ill-conditioned optimization landscapes. In contrast, for relatively shallow networks, most notably certain one-hidden-layer \ReLU\ models, training admits convex reformulations with global guarantees under appropriate settings, motivating learning paradigms that improve stability while scaling to depth. 
MGDL builds on this insight by training deep networks grade by grade: previously learned grades are frozen, and each newly added grade-wise subnetwork is composed on top of the previously learned grades and trained to fit the residual left by the current approximation, yielding a structured and interpretable hierarchical refinement process.
We develop an operator-theoretic foundation for MGDL and prove that, for any continuous target function defined on a hypercube, there exists a fixed-width multigrade \ReLU\ scheme whose residuals are pointwise nonincreasing in magnitude and converge uniformly to zero, with strict $L^p$-norm decay at every nontrivial grade for $p\in [1,\infty)$. To the best of our knowledge, this work provides the first rigorous constructive approximation guarantee showing that a grade-wise residual refinement scheme can achieve vanishing error in a fixed-width multigrade \ReLU\ architecture.

\end{abstract}


\begin{keywords}
    Multigrade Deep Learning, Structured Error Refinement, Balanced Contraction Operator, Deep Network Approximation
\end{keywords}

\vspace{12pt}
\begin{MSCcodes}
65D15, 41A46, 68T07.
\end{MSCcodes}

\section{Introduction}\label{sec:intro}

Deep neural networks have achieved remarkable empirical success in areas such as computer vision, natural language processing, and scientific computing, owing to their strong expressive power for modeling complex input-output relationships. From an approximation-theoretic perspective, this expressive capacity is now relatively well understood: sufficiently deep and wide networks can approximate broad classes of functions with high accuracy \cite{shijun:optimal:rate:in:width:and:depth,shijun:Characterized:by:Numer:Neurons,yarotsky2017}.
Despite these advances, the effective training of deep neural networks with many layers is still not fully understood.
Standard end-to-end training leads to highly nonconvex and often ill-conditioned optimization landscapes, making performance sensitive to initialization, learning rates, and architectural choices. These challenges are further exacerbated by vanishing or exploding gradients, spectral bias toward low-frequency features~\cite{rahaman2019spectral, xu2019training}, and short-term oscillations near the edge of stability \cite{arora2022understanding, cohen2021gradient}. In contrast, the optimization landscape of relatively shallow architectures is fundamentally different; in particular, 
training certain regularized one-hidden-layer \ReLU{} networks can be reformulated as convex optimization problems under appropriate settings,
yielding global optimality and strong theoretical guarantees for both training and generalization \cite{2020arXiv200210553P,FangXu2025}.

This contrast reveals a fundamental dilemma: deep networks offer strong approximation power but are difficult to optimize, whereas relatively shallow networks are easier to train but have limited expressive capacity. Motivated by this trade-off, multigrade deep learning (MGDL) was introduced in \cite{Xu2023}. The MGDL framework incrementally constructs network architectures via grade-wise training, aiming to combine the approximation advantages of deep networks with improved stability, accuracy, and interpretability.
Specifically, MGDL decomposes end-to-end optimization into a sequence of smaller subproblems, each training a relatively shallow network to approximate the residuals from previous grades. Networks learned at earlier grades are kept fixed and serve as adaptive basis functions or feature maps. This iterative refinement reduces optimization complexity while progressively improving approximation quality.
While MGDL has demonstrated promising empirical performance and has been shown to achieve grade-wise training error reduction, a fundamental theoretical question remains open: whether such grade-wise training guarantees that the approximation error converges to zero in the limit. Addressing this gap is essential for establishing a rigorous theoretical foundation for the MGDL framework. 
The main goal of this work is to show that, under appropriate architectural
conditions, this grade-wise residual-refinement mechanism admits a constructive
realization whose residuals decrease monotonically and converge to zero as the
number of grades tends to infinity.

The central idea of MGDL is \emph{grade-by-grade learning}: the network is constructed and trained one grade at a time, with  only the newly added block optimized at each grade. All previously trained blocks are frozen and serve as adaptive basis components. Rather than learning the target function in a single global optimization, MGDL decomposes the task into a sequence of residual learning subproblems with progressively finer resolution. This recursive, coarse-to-fine structure promotes improved approximation and often results in more stable training dynamics.

\begin{figure}
    \centering
    \includegraphics[width=0.985\textwidth]{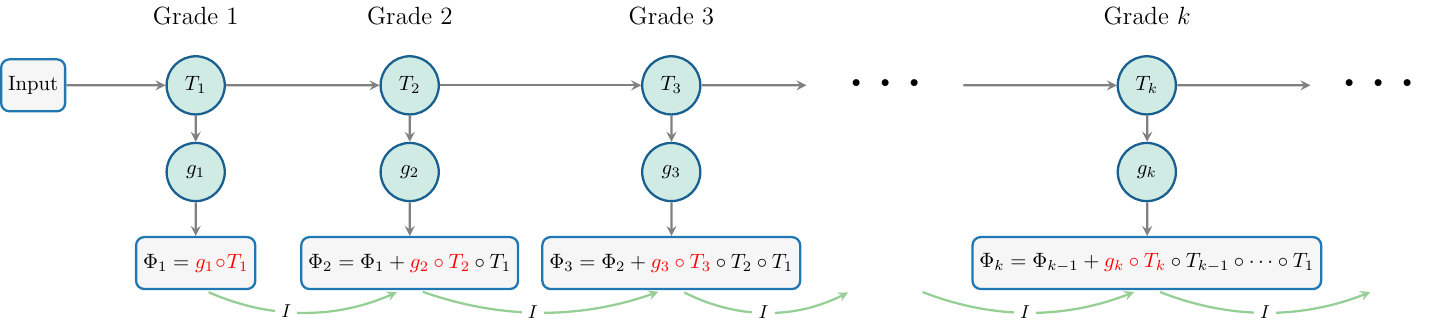}
    \caption{Overview of the  framework and its grade-wise training procedure. Training proceeds grade by grade: at each grade, only the newly added block and the output map are optimized, while all previously learned blocks are frozen and serve as adaptive basis components. The network \(\Phi_k\) approximates the target function via a recursive residual refinement process.}
    \label{fig:MGNN}
\end{figure}

As illustrated in Figure~\ref{fig:MGNN}, each grade introduces a new learnable block while keeping all previously trained components fixed. At grade \(k\), the module \(T_k\) serves as a feature transformation block, typically implemented as a shallow unit of the form \(\sigma \circ \mathcal{A}_k\), where \(\mathcal{A}_k\) is an affine map and \(\sigma\) is a fixed activation. The accompanying map \(g_k\) is a simple output operator, often affine, that converts the transformed features into a correction term for the current residual.

Concretely, at grade \(k\) we \emph{freeze} the previously learned blocks \(T_1,\dots,T_{k-1}\) and optimize only the new pair \((T_k, g_k)\). The objective is to train the composition
\[
g_k \circ T_k \circ T_{k-1} \circ \cdots \circ T_1
\]
to best approximate the residual 
\[R_{k-1} = f - \Phi_{k-1},
\] 
where \(\Phi_{k-1}\) denotes the approximation obtained from the first \(k-1\) grades. Once trained, the approximation is updated by
\[
\Phi_k = \Phi_{k-1} + g_k \circ T_k \circ \cdots \circ T_1.
\]
Thus, each grade solves a well-defined subproblem: \(T_k\) extracts new features based on all preceding transformations, and \(g_k\) fits the current residual using these features. 
MGDL adopts a stage-wise learning strategy, where each grade searches and optimizes only within a relatively low-complexity parameter (hypothesis) space; see Figure~\ref{fig:MGDL_training_path} for an illustration of this grade-wise training process. This design keeps each optimization step simple and interpretable, while raising a central theoretical issue about its overall approximation capability. 
The main goal of this work is to show that the MGDL architecture admits a constructive grade-wise refinement whose residual magnitudes decrease monotonically and converge uniformly to
zero.
A detailed algorithmic formulation of this grade-wise optimization procedure is provided in Section~\ref{sec:mgdl-algorithm}.

\begin{figure}
    \centering  \includegraphics[width=0.7\textwidth]{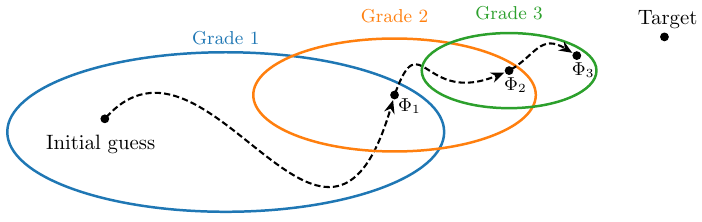}
    \caption{Overview of the MGDL grade-wise training procedure. }
    \label{fig:MGDL_training_path}
\end{figure}

In particular, when each block \(T_i\) is chosen as a one-hidden-layer transformation of the form
\(T_i = \sigma \circ \mathcal{A}_i\), and each output map \(g_i\) is affine, denoted by
\(\mathcal{A}_i^{\mathrm{out}}\), the grade-wise update takes the simplified form
\[
\Phi_k
=
\Phi_{k-1}
+
\mathcal{A}_k^{\mathrm{out}}
\circ
\sigma
\circ
\mathcal{A}_k
\circ
\cdots
\circ
\sigma
\circ
\mathcal{A}_1.
\]
Consequently, the network obtained after \(k\) grades admits the explicit representation
\begin{equation}
\label{eq:def:Phi:k}
    \Phi_k
=
\sum_{\ell=1}^k
\mathcal{A}_{\ell}^{\mathrm{out}}
\circ
\ocomp_{i=1}^{\ell}
\bigl(\sigma \circ \mathcal{A}_i\bigr).
\end{equation}
Here, \(\ocomp\) denotes ordered composition: for a collection of maps
\(\{f_i\}_{i=n}^m\), we define
    \[
\ocomp_{i=n}^m f_i
\;\coloneqq\;
f_m  \circ \cdots \circ f_{n+1} \circ f_n.
\]


Viewed more broadly, MGDL departs fundamentally from standard end-to-end
training by organizing deep network construction as a sequence of
grade-wise residual refinements.
At each grade, the newly introduced block is composed with all previously
learned transformations, which remain frozen during training.
This design preserves the full compositional expressive power of deep
networks while decomposing learning into a sequence of controlled and
interpretable subproblems.
In contrast to classical approaches based on linear combinations of fixed basis functions, MGDL introduces adaptive, data-driven basis functions at each grade. These components are realized through additional layer compositions built upon frozen earlier transformations.
As a result, the representational capacity grows in a structured and hierarchical manner, enabling progressively richer feature extraction while maintaining training stability and interpretability.

From both optimization and approximation perspectives, MGDL implements a
divide-and-conquer strategy.
Each grade solves a relatively shallow residual learning problem, thereby reducing nonconvex
interactions across layers, improving conditioning, and inducing a natural
coarse-to-fine refinement mechanism in which dominant structures are
captured first and finer details are resolved progressively.
This hierarchical structure also facilitates theoretical analysis and, in certain
settings, allows multigrade training to be interpreted as a sequence of
simpler optimization problems, including \emph{convex} formulations for shallow
\ReLU{} networks~\cite{2020arXiv200210553P,FangXu2025}.

Empirically, MGDL has demonstrated faster convergence, improved numerical
stability, and enhanced approximation accuracy across regression, imaging, integral equations,
and PDE-based applications~\cite{Jiang-Xu2025,Xu2023,FangXu2024,FangXu2025,XuZeng2023}.
Despite these empirical successes and several theoretical observations concerning monotone training-error decay and reduced sensitivity of optimization convergence to learning rates, existing studies of MGDL remain largely experimental from an approximation-theoretic perspective.
In particular, there is currently no rigorous theoretical framework that characterizes structural
conditions under which the residual converges uniformly to zero.
Establishing such guarantees is essential for elevating MGDL from a
practical training paradigm to a principled approximation framework.

The main objective of this paper is to establish a rigorous
approximation-theoretic foundation for MGDL. We formulate the grade-wise
refinement process from an operator-theoretic perspective and prove that, for a
concrete fixed-width multigrade \ReLU{} architecture, the resulting
approximants converge uniformly to the target function, while the residuals
decrease monotonically across grades.
More precisely, our main result, Theorem~\ref{thm:main:mgdl}, shows that for every
\(f\in C([0,1]^d)\), there exist affine maps \(\mathcal{A}_i\) and output
affine maps \(\mathcal{A}_i^{\mathrm{out}}\), with all intermediate dimensions
bounded by \(5d+1\), such that the multigrade approximants \(\Phi_k\) for
\(k=1,2,\ldots\), defined in \eqref{eq:def:Phi:k} with \(\sigma=\ReLU\), converge
uniformly to \(f\) on \([0,1]^d\). More concretely, if
\(R_k \coloneqq f-\Phi_k\)
denotes the residual after \(k\) grades, then \(R_k\to 0\) uniformly on
\([0,1]^d\). Moreover, the residuals are pointwise nonincreasing in magnitude: $|R_{k+1}(\bmx)| \le |R_k(\bmx)|$
for all $k$ and $\bmx\in[0,1]^d$,
and, for every \(p\in [1,\infty)\), their \(L^p\) norms decrease strictly at each
grade unless the residual is already identically zero.

The remainder of this paper is organized as follows.
Section~\ref{sec:mgdl-theory} develops the residual-refinement perspective underlying MGDL. It establishes the main approximation and contraction results, explains how successive grades reduce the residual left by the previous grades, and relates this viewpoint to dictionary-based approximation and existing neural-network approximation theory.
Section~\ref{sec:mgdl-framework} then describes
the practical MGDL algorithm: Section~\ref{sec:mgdl-algorithm} formulates the
grade-wise optimization procedure, while Section~\ref{sec:mgdl-experiments}
reports numerical experiments illustrating the residual-refinement mechanism.
Section~\ref{sec:proof:thm:main:mgdl} proves the main approximation theorem
by reducing it to the one-step contraction result
Theorem~\ref{thm:main:mgdl:one:op}. The proof of this auxiliary theorem is
given in Section~\ref{sec:proof:thm:main:mgdl:one:op}, where the contraction
operator is constructed and realized by fixed-width \ReLU{} networks.
Finally, Section~\ref{sec:conclusion} concludes the paper and discusses future
directions.

\section{Residual refinement}\label{sec:mgdl-theory}

    Residual refinement refers to an iterative approximation strategy in which each
new grade is constructed to correct the residual left by the previous grades.
This section develops MGDL as such a structured residual-refinement framework.
Section~\ref{sec:main-results} states the main fixed-width approximation
theorem and shows that the resulting residuals converge uniformly to zero with
a monotone contraction property. Section~\ref{sec:mgdl-dictionary} interprets
this refinement process from the viewpoint of dictionary-based approximation,
where grade-wise learning corresponds to an adaptive expansion of a
compositional dictionary. Finally, Section~\ref{sec:related-work} reviews
related work on neural network approximation, frequency-based learnability,
and hierarchical training schemes, and clarifies how the present residual
refinement perspective differs from existing approaches.

\subsection{Monotone residual refinement in MGDL}
\label{sec:main-results}

Throughout the paper, we use $\aff_{\le n}$ to denote the collection of all affine maps whose input and output dimensions are at most $n$, that is,
    \[
\aff_{\le n} \coloneqq 
\bigcup_{1 \le k \le n}  \, \bigcup_{1 \le m \le n}   \aff(k,m),
\]
where $\aff(k,m)$ denotes the set of affine transformations from $\mathbb{R}^k$ to $\mathbb{R}^m$, namely,
\[
\aff(k,m)
= \left\{
\bm{x} \mapsto \bm{W} \bm{x} + \bm{b} \ :\ \bm{W} \in \mathbb{R}^{m \times k},\ \bm{b} \in \mathbb{R}^m
\right\}.
\]
As usual, $\mathbb{N}$, \( \mathbb{Z} \), and $\mathbb{R}$ denote the sets of natural numbers (including $0$), integers, and real numbers, respectively, and we write
\(\mathbb{N}^+ \coloneqq \mathbb{N} \setminus \{0\} = \{1,2,3,\dots\}.\)
For a domain $\Omega \subseteq \mathbb{R}^d$, $C(\Omega)$ denotes the space of real-valued continuous functions on $\Omega$.
    For $\bmx=(x_1,\ldots,x_d)\in\R^d$, we write
$\bmx{[n:m]}=(x_n,\ldots,x_m)$ and $\bmx{[n]}=x_n$
for $1\le n\le m\le d$, e.g., if $\bmx\in\R^5$, then $(5\bmx){[2:3]}=(5x_2,5x_3)$.
The symbol
$\sigma:\mathbb{R}\to\mathbb{R}$ denotes a fixed activation function, with $\ReLU$
being the primary example considered in this work. 
We write $a_n \searrow a$ to indicate that the sequence is nonincreasing,
that is, $a_{n+1}\le a_n$ for all $n$, and converges to $a$, with the
decrease called strict if $a_{n+1}<a_n$ for all $n$.
Additional notation will be introduced as needed.

We now state our main approximation theorem for MGDL.





\begin{theorem}\label{thm:main:mgdl}
Given $f \in C([0,1]^d)$, there exist affine maps 
$\mathcal{A}_{i},\, \mathcal{A}_{i}^{\mathrm{out}} \in \aff_{\le 5d+1}$ for all
$i\in\mathbb{N}^+$ such that
the sequence of multigrade
approximations
\[
\Phi_k
\;\coloneqq\;
\sum_{\ell=1}^k
\mathcal{A}_{\ell}^{\mathrm{out}}
\circ
\ocomp_{i=1}^{\ell}
\bigl(\ReLU \circ \mathcal{A}_i\bigr) 
\quad \tn{for all}  \   k\in \N^+
 \]
generates residuals \(R_k \coloneqq f - \Phi_k\) with the following properties: 
\begin{enumerate}[label=(\roman*)]
\item 
For every $\bmx \in [0,1]^d$,  
$|R_k(\bmx)| \searrow 0$  and 
$\|R_k\|_{L^\infty([0,1]^d)} \searrow 0$ as $k \to \infty$.

\item 
Given any $p\in[1,\infty)$,
$\|R_k\|_{L^p([0,1]^d)}\searrow 0$ as $k\to\infty$,
with strict decrease whenever $R_k\not\equiv 0$, i.e., 
 for any $k\in \N^+$, 
either $R_{k+1}\equiv R_k\equiv0$ on $[0,1]^d$  or $\|R_{k+1}\|_{L^p([0,1]^d)}< \|R_k\|_{L^p([0,1]^d)}$.
\end{enumerate}
\end{theorem}

The proof of Theorem~\ref{thm:main:mgdl} is presented in
Section~\ref{sec:proof:thm:main:mgdl}. We emphasize that the proof is fully
constructive. More precisely, after a finite number of grades, determined by
the complexity of the target function and a prescribed accuracy parameter $\eps\in \big(0,\tfrac{1}{1+2^{d}}\big)$,
the residual admits an $L^\infty$ contraction by a factor of $1-\varepsilon$.

\begin{corollary}\label{coro:mgdl:residual}
Given $f\in C([0,1]^d)$ and $\eps\in \big(0,\tfrac{1}{1+2^{d}}\big)$, there exists a subsequence $\{k_j\}_{j=1}^\infty$, constructed explicitly from $f$ and $\eps$,
such that the residuals $R_k$ in Theorem~\ref{thm:main:mgdl} satisfy
\begin{equation*}
    \|R_{k_{j+1}}\|_{L^\infty([0,1]^d)}
    \le (1-\eps)\, \|R_{k_j}\|_{L^\infty([0,1]^d)}\quad \tn{for all }
     j\in \N^+.
\end{equation*}
\end{corollary}

We stress that the subsequence
$\{k_j\}_{j=1}^\infty$ in Corollary~\ref{coro:mgdl:residual} is obtained through an explicit construction depending on $f$ and $\eps$, rather
than by a purely existential argument, as will be evident from the
proof of Theorem~\ref{thm:main:mgdl}.
To the best of our knowledge,
these results provide the first rigorous
theoretical foundation for MGDL as a principled
framework for structured, grade-wise error refinement.
They demonstrate that deep networks can be constructed incrementally through a
sequence of residual corrections whose approximation errors decrease
monotonically across successive grades, with strict decrease in every
$L^p$-norm for $1 \le p < \infty$ whenever the current residual is nonzero.
This establishes a precise mechanism by which depth enables
progressive refinement of approximation accuracy and offers a concrete connection
between MGDL training and classical approximation theory.

We emphasize that our analysis assumes an affine output map $g$ and
grade blocks $T_i$ formed by composing an affine transformation with a
fixed activation function. This setting covers a broad and practically
relevant class of neural network architectures, but the underlying
ideas are not restricted to this configuration.
For example, consider replacing each one-hidden-layer block
$\ReLU \circ \mathcal{A}_i$ with a two-hidden-layer module
$T_i = \ReLU \circ \mathcal{A}_{2i} \circ \ReLU \circ \mathcal{A}_{2i-1}$.
The original construction can be embedded in this architecture by
choosing $\mathcal{A}_{2i}$ as the identity. Indeed, the output of the
preceding \ReLU\ layer is coordinatewise nonnegative, so
$(\ReLU \circ \mathcal{A}_{2i})(\boldsymbol{z})=\boldsymbol{z}$
on the relevant feature range. Thus, the additional layer leaves
both the width and the represented function unchanged.
More generally, richer output maps and block structures may admit
similar constructions, although their analysis would require
additional arguments. A systematic treatment of these extensions
is beyond the scope of the present work and is left for future research.

Finally, we clarify the scope and interpretation of our theoretical result. Theorem~\ref{thm:main:mgdl} is an approximation-theoretic
construction result. It establishes the existence of fixed-width multigrade
\ReLU{} corrections whose residuals decrease monotonically and converge
uniformly to zero, with the required affine and output maps constructed through
localized cutoff functions and contraction arguments. This should not be
interpreted as a convergence theorem for gradient descent, Adam, or other
practical optimizers. The numerical experiments in Section~3 only illustrate
that the proposed grade-wise training can effectively realize this
residual-refinement mechanism in practice, while its optimization dynamics are
left for future study.
To further elucidate the implications of
Theorem~\ref{thm:main:mgdl}, we next interpret the MGDL framework from the
perspective of dictionary-based approximation.

\subsection{Understanding MGDL from dictionary-based approximation}
\label{sec:mgdl-dictionary}

From the viewpoint of dictionary-based approximation, the MGDL refinement
process admits a natural interpretation in terms of basis expansions and sparse approximation.
These paradigms provide a conceptual framework for understanding the mechanism
of progressive error reduction that underlies grade-wise training.

\subsubsection*{Basis expansions}
Classical basis expansions provide a conceptually straightforward framework
for function approximation, with Fourier-type expansions serving as a canonical
example.
Given a prescribed and fixed collection of basis functions, approximation is achieved
by selecting basis elements, often in a predetermined or hierarchical order, and
forming linear combinations to approximate the target function.
Typically, the approximation accuracy improves as additional basis elements are included. 
However, because the basis is fixed a priori and does not adapt to the specific
structure of the target function, such approaches can be inefficient when the target
exhibits complex, anisotropic, or highly localized features.

\subsubsection*{Sparse approximation}
Sparse approximation adopts a more flexible, but technically more involved, approach. Instead of relying on a basis, it employs a \emph{dictionary}, typically a redundant system, such as multiresolution-analysis-based wavelet tight frames in $L^2(\mathbb{R}^d)$ \cite{RON1997408}.
Given a dictionary $\mathcal{D} = \{\bme_i\}_{i \in \mathcal{I}}$, the goal is to
approximate a target function using only a small number of dictionary elements.
A standard greedy strategy proceeds by iteratively reducing the residual.
Starting from $R_0 = f$, one selects at each iteration an atom-coefficient pair
by solving
\[
(\alpha_{k+1}, \bme_{i_{k+1}})
\;\in\;
\argmin_{\alpha \in \mathbb{R},\, \bme \in \mathcal{D}}
\| R_k - \alpha \bme \|
\]
with respect to a chosen norm, and updates the residual via
$R_{k+1} = R_k - \alpha_{k+1} \bme_{i_{k+1}}$.
    Here and below, the notation ``\(\argmin\)''  is used in a schematic sense in the optimization formulations. Since an exact minimizer may not exist without additional assumptions, it should be understood, when necessary, as an approximate minimizer or a near-minimizing choice for the displayed objective.
Each iteration seeks a dictionary element that explains the current residual. Exact
minimization, or an approximate update accepted only when it does not increase the residual norm, yields
monotone refinement.
While sparse approximation is highly effective, its performance typically depends
on the availability of a sufficiently rich dictionary, often requiring a large
number of atoms to achieve high accuracy.

\subsubsection*{Neural networks with end-to-end training}
From the sparse approximation viewpoint, standard end-to-end neural network
training can be interpreted as a single-step approximation problem over a
highly expressive dictionary.
Let $\calT$ be the class of admissible network blocks, and let $\mathcal G$
be the prescribed class of output maps.
Specifically, one may regard the neural network dictionary as
\[
\mathcal{D}_{\mathrm{NN}}
=
\bigl\{
T_m \circ \cdots \circ T_1
:\;
T_1, \dots, T_m \in \calT
\bigr\},
\]
for a chosen sufficiently large $m \in \mathbb{N}^+$.
This dictionary consists of a large family of parameterized compositional maps.
End-to-end training then amounts to selecting a single element from this
dictionary together with an output map $g \in \mathcal{G}$ by solving
\[
(g^\ast, \bme^\ast)
\;\in\;
\argmin_{g \in \mathcal{G},\, \bme \in \mathcal{D}_{\mathrm{NN}}}
\| f - g \circ \bme \|.
\]
The resulting approximation takes the form
\[
f \approx g^\ast \circ \bme^\ast
= g^\ast \circ T_m^\ast \circ \cdots \circ T_1^\ast .
\]
Although the dictionary $\mathcal{D}_{\mathrm{NN}}$ is, in principle, highly expressive, this formulation highlights the intrinsic difficulty of
end-to-end learning.
Selecting a single, highly structured element from an extremely large and
nonconvex dictionary is inherently challenging, as approximation
errors must be coordinated across all layers simultaneously.
This coupling often leads to severe optimization difficulties and unstable training
dynamics, despite the strong representational capacity of the underlying
dictionary.

\subsubsection*{MGDL viewpoint}
MGDL combines key advantages of sparse approximation and neural network-based
representations.
On the one hand, it retains a highly expressive dictionary generated by deep
compositions.
On the other hand, it replaces the single-shot selection of an entire network
with a sequence of simple, structured greedy steps.

Starting from $R_0=f$, let $\bme_0$ be the identity map. MGDL constructs 
a sequence of adaptive compositional dictionaries
in a grade-wise manner. Let
$\calT$ denote the class of admissible network blocks and let $\mathcal G$
denote the chosen class of output maps. Given the feature representation
$\bme_k$ learned up to grade $k$, the dictionary available at the next grade is
defined by
\[
    \mathcal D_{\mathrm{MGDL}}(k+1)
    :=
    \bigl\{
        T\circ\bme_k : T\in\calT
    \bigr\}.
\]
In particular, since $\bme_0$ is the identity map, the initial dictionary is
\[
    \mathcal D_{\mathrm{MGDL}}(1)
    =
    \bigl\{
        T_1 : T_1\in\calT
    \bigr\}.
\]

At grade $k+1$, MGDL selects a new block $T_{k+1}\in\calT$ together with an
output map $g_{k+1}\in\mathcal G$ by solving
\[
    (g_{k+1},T_{k+1})
    \in
    \operatorname*{arg\,min}_{g\in\mathcal G,\ T\in\calT}
    \bigl\|
        R_k-g\circ T\circ\bme_k
    \bigr\|.
\]
Equivalently, the selected atom is
\[
    \bme_{k+1}
    :=
    T_{k+1}\circ\bme_k
    \in
    \mathcal D_{\mathrm{MGDL}}(k+1),
\]
and the residual is updated by
\[
    R_{k+1}
    :=
    R_k-g_{k+1}\circ\bme_{k+1}.
\]
Thus, at each grade, only the newly introduced block $T_{k+1}$ and the
corresponding output map $g_{k+1}$ are optimized, while all previously learned
blocks remain fixed. Consequently, each step involves a relatively simple
optimization problem, whereas the overall approximation benefits from 
increasingly deep compositional representations.
This stage-wise structure reduces the number of parameters
optimized jointly at each grade compared with standard
end-to-end training.

Finally, it is important to note that after $k$ grades have been constructed, the learned feature representation is
\[
    \bme_k
    =
    T_k\circ T_{k-1}\circ\cdots\circ T_1 .
\]
If this representation is kept fixed, then further optimization can only be
performed within the output space
\[
    \mathcal H_k^{\mathrm{out}}
    :=
    \bigl\{
        g\circ\bme_k : g\in\mathcal G
    \bigr\}.
\]
Since the $k$-th grade has already been trained to approximate the residual
$R_{k-1}$ using this fixed feature representation, merely re-optimizing the
output map is expected to provide only limited additional improvement. MGDL
overcomes this limitation by appending a new admissible block. More precisely,
the next grade searches the compositional class
\[
    \mathcal H_{k+1}(\bme_k)
    :=
    \bigl\{
        g\circ T_{k+1}\circ\bme_k :
        T_{k+1}\in\calT,\ g\in\mathcal G
    \bigr\}.
\]
Hence 
each new grade can generate new compositional correction directions
along which the residual can be reduced. In the idealized case of
exact optimization, provided that the zero map belongs to $\mathcal G$, the
residual norms form a nonincreasing sequence. This mechanism explains why each
new grade may lead to a substantial error decrease despite the greedy nature of
the construction.

The dictionary viewpoint above clarifies the structural role of MGDL:
each new grade generates an adaptive compositional atom for residual
correction.
We next compare this perspective with existing
work on neural network expressivity, frequency-dependent learnability, and
hierarchical residual training.

\subsection{Related work}
\label{sec:related-work}

The present work is closely related to several active research directions in
neural network approximation theory and learning theory.
Below, we briefly review representative developments on expressive power,
frequency-based perspectives on learnability, and hierarchical or multigrade
training strategies, and clarify how our results differ from and complement
existing approaches.

\subsubsection*{Expressive power of neural networks}
The expressive capacity of neural networks has been a central topic in approximation theory and learning theory.
Classical universal approximation results
\cite{Cybenko1989ApproximationBS,HORNIK1989359,HORNIK1991251}
establish that even shallow neural networks can approximate arbitrary continuous functions on compact domains.
While foundational, these results are qualitative in nature and do not quantify how network size, depth, or architectural constraints affect approximation efficiency.
Subsequent research has significantly refined this theory by deriving sharp approximation rates for specific function classes and by elucidating the roles of depth, width, and compositional structure in determining representational efficiency
\cite{yarotsky18a,yarotsky2017,B_lcskei_2019,
2019arXiv190501208G,2019arXiv190207896G,MO,
shijun:Characterized:by:Numer:Neurons,shijun:smooth:functions,
shijun:arbitrary:error:with:fixed:size,shijun:2023:beyond:ReLU:to:diverse:actfun,
yarotsky:2019:06,SIEGEL20221,LiLinShen2023DynamicalSystems,
Bao2019ApproximationAO,ZHOU2019,shijun:optimal:rate:in:width:and:depth}. 

More recently, deep learning in reproducing kernel Banach spaces (RKBS) was introduced in~\cite{Wang_Xu_Yan2025}, where deep neural networks are interpreted as asymmetric kernels with respect to the input and parameter variables. This viewpoint recasts deep learning as a form of kernel-based learning in RKBS, extending the classical reproducing kernel Hilbert space framework, and yields representer theorems for deep architectures.
In parallel, a growing body of work has explored novel architectures and activation functions aimed at enhancing approximation power or reducing model complexity
\cite{shijun:intrinsic:parameters,shijun:RCNet,WANG2025107258,
shijun:floor:relu,shijun:three:layers,ZZZZ-25-FMMNN,
shijun:net:arc:beyond:width:depth}.
Despite this substantial progress, most of the existing literature focuses primarily on representational capacity.
Fundamental questions concerning \emph{learnability}, optimization dynamics, and the gap between expressivity and effective training remain largely unresolved
\cite{pmlr-v70-shalev-shwartz17a,pmlr-v80-safran18a,JMLR:v20:18-674,ZZZZ-23}.

\subsubsection*{Frequency-based perspectives on learnability}
    Motivated by the gap between expressivity and practical trainability, a
complementary line of research examines neural network learning from a
frequency-based perspective
\cite{ZZZZ-23,ZZZZ-24-MMNN,xu2019frequency,luo2019theory,rahaman2019spectral,pmlr-v119-basri20a,2025arXiv251102625L}.
A central concept in this direction is the \emph{frequency principle}
\cite{luo2019theory,xu2019frequency}, which posits that neural networks tend to
fit low-frequency components of a target function earlier than high-frequency
ones during training. This phenomenon has been empirically observed across a
wide range of architectures and learning tasks, and has inspired theoretical
studies linking optimization dynamics, spectral bias, and generalization.

This viewpoint is closely related to the grade-wise residual refinement
mechanism studied in this paper. Early grades often capture dominant
coarse-scale structures, while later grades may focus more on finer-scale or
oscillatory residual features. However, MGDL does not rely on a strict monotone
frequency-separation assumption, since later residuals may still contain
low-frequency components that were not fully learned earlier. Moreover,
high-frequency structures in the original input variables may become simpler,
or effectively lower-frequency, after suitable feature transformations or
compositions; this phenomenon has also been observed in the context of
multi-component and multi-layer neural networks \cite{ZZZZ-24-MMNN,ZZZZ-25-FMMNN}. In MGDL, later grades are
built upon the previously learned and frozen transformations, rather than being
trained from the raw input alone. Thus, later blocks can exploit these learned
feature coordinates to approximate finer-scale residuals, without necessarily
requiring larger widths at later grades.

Thus, MGDL should be viewed as an adaptive coarse-to-fine residual refinement
mechanism rather than a rigid low-to-high frequency decomposition. Nevertheless,
existing results are largely qualitative or problem-dependent, and a rigorous,
architecture-level theory that  characterizes frequency-dependent
learnability is still lacking.

\subsubsection*{Hierarchical, residual, and multigrade training strategies}

A closely related line of work studies architectural and algorithmic mechanisms that promote stable optimization and progressive error refinement.
Residual learning
\cite{7780459},
layerwise or greedy training schemes
\cite{6796673,Bengio2007},
and hierarchical forecasting or refinement architectures
\cite{oreshkin2020nbeats,NEURIPS2023_1d5a9286,9614997}
have all demonstrated empirical success in improving convergence behavior and mitigating optimization difficulties.
Within this context, the MGDL framework was introduced empirically in \cite{Xu2023} to enhance training stability, accelerate convergence, and alleviate optimization pathologies such as vanishing gradients. 
A specialized variant of MGDL was later proposed in~\cite{Xu2025} in the form of the successive affine learning model, where each affine transformation is first learned by solving a quadratic or convex optimization problem, and the activation function is applied only after the weight matrix and bias vector of the current layer are determined.
While related coarse-to-fine and multilevel strategies have shown strong performance across regression, imaging, integral equation and PDE-solving tasks
\cite{FangXu2024,FangXu2025,JiangXu2024,Jiang-Xu2025,Xu2023}, existing analyses remain largely empirical. Consequently, they lack rigorous guarantees that grade-wise residuals will converge uniformly to zero within concrete neural network architectures.

Despite promising empirical evidence for hierarchical and residual methods, none of the aforementioned approaches, including the original MGDL study \cite{Xu2023}, establishes a rigorous approximation theory that guarantees
\emph{uniform convergence} in an explicit neural network construction.
This work addresses this gap in the setting of fixed-width multigrade \ReLU{} residual-sum architectures.
We develop an operator-theoretic formulation of MGDL and identify the structural conditions under which a fixed-width, multigrade \ReLU{} architecture admits a constructive grade-wise residual-refinement scheme. 
We prove that this construction yields pointwise domination of residual magnitudes, strict \(L^p\)-norm decay at every nontrivial grade for all \(p\in [1,\infty)\), and uniform convergence of the residuals to zero.
These results provide the first theoretical foundation for MGDL as a principled framework for structured error refinement, bridging the gap between classical expressivity theory and modern hierarchical learning.

\section{MGDL algorithm and experiments}
\label{sec:mgdl-framework}

This section connects the abstract residual-refinement theory developed in
Section~\ref{sec:mgdl-theory} with practical MGDL training. In
Section~\ref{sec:mgdl-algorithm}, we formulate the grade-wise optimization
procedure at both the function level and the data level, emphasizing how each
new block is trained on the current residual while previous blocks are frozen.
In Section~\ref{sec:mgdl-experiments}, we present one- and two-dimensional
experiments that illustrate a residual-refinement pattern consistent with the
construction in
Theorem~\ref{thm:main:mgdl} and compare MGDL with standard end-to-end training.

\subsection{Grade-wise optimization}
\label{sec:mgdl-algorithm}

We begin by formulating the MGDL learning process at the function level. Let the target be a function \(f\), and let \(T_i\) denote a parameterized network block, for example \(T_i = \sigma \circ \mathcal{A}_i\) with an activation function \(\sigma\) and an affine map \(\mathcal{A}_i\). Let \(g_i\) denote a parameterized output map, such as an affine transformation that ensures dimensional consistency with the target.

As illustrated in Figure~\ref{fig:MGNN}, MGDL constructs the network grade by grade.
\begin{itemize}
    \item \textbf{Grade 1.}
    Train 
    \(\Phi_1 = g_1 \circ T_1\)
    to approximate \(f\). The first residual is
    \(R_1 = f - \Phi_1.\)

    \item \textbf{Grade 2.}
    Freeze \(T_1\) and train a new block \(T_2\) together with an output map \(g_2\) so that
    \(
        g_2 \circ T_2 \circ T_1
    \)
    approximates the residual \(R_1\). The updated approximation is
    \[
        \Phi_2 = \Phi_1 + g_2 \circ T_2 \circ T_1,
    \]
    and the new residual is given by \(R_2 = f - \Phi_2\).

    \item \textbf{Higher grades.}
    At grade \(k \ge 3\), we freeze \(T_1,\dots,T_{k-1}\), introduce a new block \(T_k\) and an output map \(g_k\), and train
    \(
        g_k \circ T_k \circ T_{k-1} \circ \cdots \circ T_1
    \)
    to approximate the residual \(R_{k-1} = f - \Phi_{k-1}\). The grade \(k\) approximant becomes
    \[
        \Phi_k = \Phi_{k-1} + g_k \circ T_k \circ T_{k-1} \circ \cdots \circ T_1,
    \]
    and the residual is updated to \(R_k = f - \Phi_k\).
\end{itemize}

After \(k\) grades, the overall MGDL approximant can be written as
\begin{equation*}
    \Phi_k
    = \sum_{i=1}^{k}
      g_i \circ T_i \circ T_{i-1} \circ \cdots \circ T_1,
\end{equation*}
where \(T_1\) acts on the original input. As \(k\) increases, the approximation is refined recursively, and the central question is whether the residuals \(R_k\) can be guaranteed to decrease monotonically and converge to zero under suitable architectural choices. Theorem~\ref{thm:main:mgdl} provides an affirmative answer for a class of fixed-width multigrade \ReLU{} networks.

The same construction can be expressed directly in terms of the training data.
Suppose we are given samples \(\{(\bmx_i, f(\bmx_i))\}_{i=1}^N\). For convenience, set
\[
    \bmx_i^{(1)} \coloneqq \bmx_i,\quad  y_i^{(1)} \coloneqq f(\bmx_i^{(1)}).
\]
The first grade learns the first approximation
$\Phi_1=g_1\circ T_1$ by solving
\[
    \min_{g_1, T_1}
    \frac{1}{N}\sum_{i=1}^N
    \calL\bigl(g_1 \circ T_1(\bmx_i^{(1)}),\, y_i^{(1)}\bigr),
\]
where \(\calL(\cdot,\cdot)\) denotes the loss function.
We then define the transformed inputs and residual outputs
\[
    \bmx_i^{(2)} \coloneqq T_1(\bmx_i^{(1)}),\quad 
    y_i^{(2)} \coloneqq y_i^{(1)} - g_1 \circ T_1(\bmx_i^{(1)}).
\]
At grade \(2\), the block \(T_1\) is fixed, and we solve
\[
    \min_{g_2, T_2}
    \frac{1}{N}\sum_{i=1}^N
    \calL\bigl(g_2 \circ T_2(\bmx_i^{(2)}),\, y_i^{(2)}\bigr),
\]
which yields the second approximation \(\Phi_2 = g_2 \circ T_2\circ T_1 + g_1\circ T_1\).
In general, at grade $k$, suppose that
$T_1,\ldots,T_{k-1}$ have already been learned. We define
\[
    \bmx_i^{(k)} \coloneqq T_{k-1}\circ \cdots \circ T_1(\bmx_i),\quad 
    y_i^{(k)} \coloneqq y_i^{(k-1)} - g_{k-1} \circ T_{k-1}(\bmx_i^{(k-1)}),
\]
freeze \(T_1, \dots, T_{k-1}\), and solve the grade-wise problem
\[
    \min_{g_k, T_k}
    \frac{1}{N}\sum_{i=1}^N
    \calL\bigl(g_k \circ T_k(\bmx_i^{(k)}),\, y_i^{(k)}\bigr).
\]
Thus, at grade \(k\) the learning task focuses on capturing the current residual through \(g_k \circ T_k\) using the data \(\{(\bmx_i^{(k)}, y_i^{(k)})\}_{i=1}^N\), rather than reoptimizing all layers simultaneously as in end-to-end training. This grade-wise formulation makes the recursive refinement structure explicit and aligns directly with the operator-theoretic framework underlying Theorem~\ref{thm:main:mgdl}.

Finally, we clarify the notation used in the grade-wise training formulation. The data
\(\{(\bmx_i^{(k)},y_i^{(k)})\}_{i=1}^N\) introduced at grade \(k\) should not be
understood as a new or disjoint data set. It is only a convenient way to describe
the effect of the previously trained network blocks as a feature transformation of
the original samples, together with the corresponding update of the residual
labels. In the actual implementation, one does not need to explicitly construct a
new data set at each grade. It suffices to freeze the previously trained blocks and
train the next block on the same original samples, with the loss evaluated through
the updated residual.

\subsection{Numerical experiments}
\label{sec:mgdl-experiments}

The purpose of this subsection is to illustrate the grade-wise
residual-refinement mechanism underlying the theoretical results
in Section~\ref{sec:main-results}.
Through numerical experiments in one and two spatial dimensions, we examine
whether the residual error decreases consistently as the grade index increases,
and whether the activation of each new grade yields a substantial and stable
improvement in approximation accuracy.
These empirical observations directly reflect the grade-wise residual refinement
principle developed in the theoretical analysis.
As a supplementary reference, comparisons with standard end-to-end training are
reported at the end of this section.

\subsubsection*{Target functions}

We consider target functions with pronounced high-frequency components and
nonlinear couplings.
In one dimension, the target function is
\[
f_1(x)=\sin(32\pi x)-0.5\cos(16\pi x^2).
\]
In two dimensions, we use the coupled oscillatory function
\[
f_2(x_1,x_2)=\sum_{i=1}^2\sum_{j=1}^2 a_{i,j}
\sin(b_i x_i+c_{i,j}x_ix_j)\bigl|\cos(b_j x_j+d_{i,j}x_i^2)\bigr|,
\]
where
\[
(a_{i,j})=\begin{bmatrix*}0.3&0.2\\0.2&0.3\end{bmatrix*},\quad 
(b_i)=\begin{bmatrix*}12\pi\\8\pi\end{bmatrix*},\quad 
(c_{i,j})=\begin{bmatrix*}4\pi&12\pi\\6\pi&10\pi\end{bmatrix*},\quad 
(d_{i,j})=\begin{bmatrix*}14\pi&12\pi\\8\pi&10\pi\end{bmatrix*}.
\]
The target functions \(f_1\) and \(f_2\) are illustrated in
Figure~\ref{fig:vsFCNN:f1D2D}.
The experimental domains \([-1,1]\) and \([-1,1]^2\) can be affinely
rescaled to \([0,1]^d\). The experiments use learned parameters and are intended as illustrative
examples rather than direct instantiations of the constructive proof.

\begin{figure}
\centering
\begin{minipage}[b]{0.88\linewidth}
\begin{subfigure}[b]{0.365\linewidth}
    \centering	
    \includegraphics[width=0.925\textwidth]{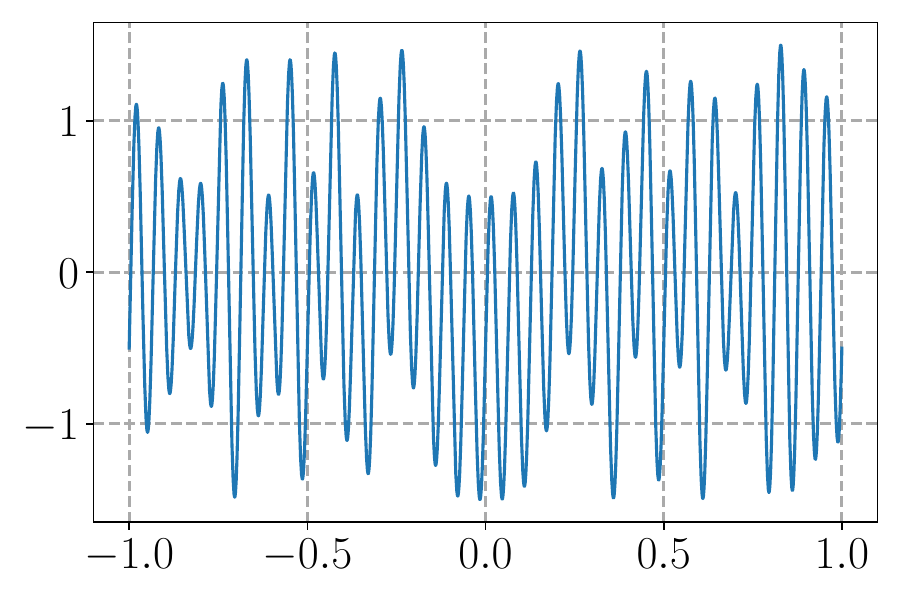}
    \caption{$f_1$.}
\end{subfigure}
\hfill
\begin{subfigure}[b]{0.5519\linewidth}
    \centering	
    \begin{subfigure}[c]{0.48\textwidth}
        \includegraphics[height=0.84\textwidth]{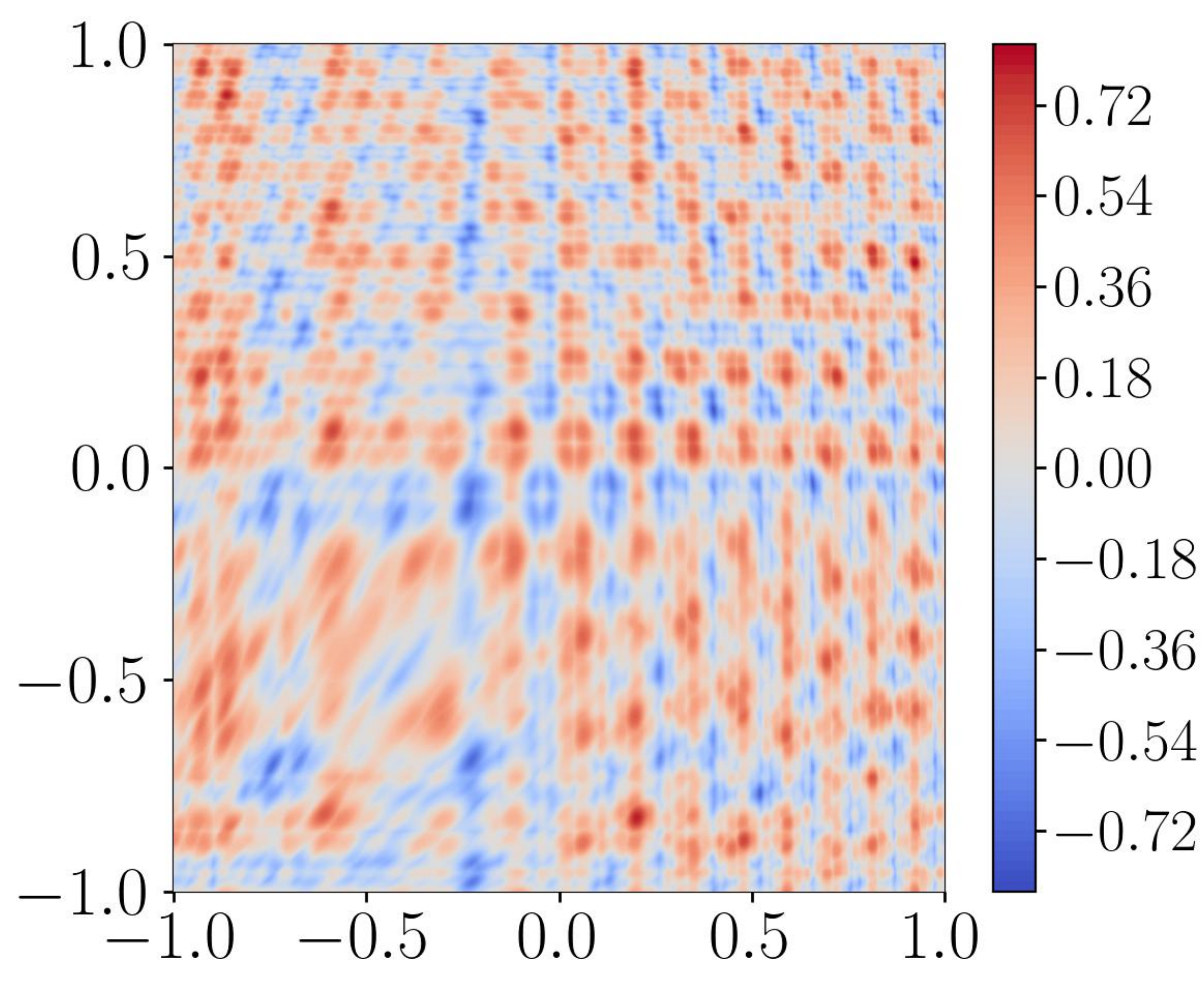}
    \end{subfigure}
    \hfill
    \begin{subfigure}[c]{0.48\textwidth}
        \includegraphics[height=0.85\textwidth]{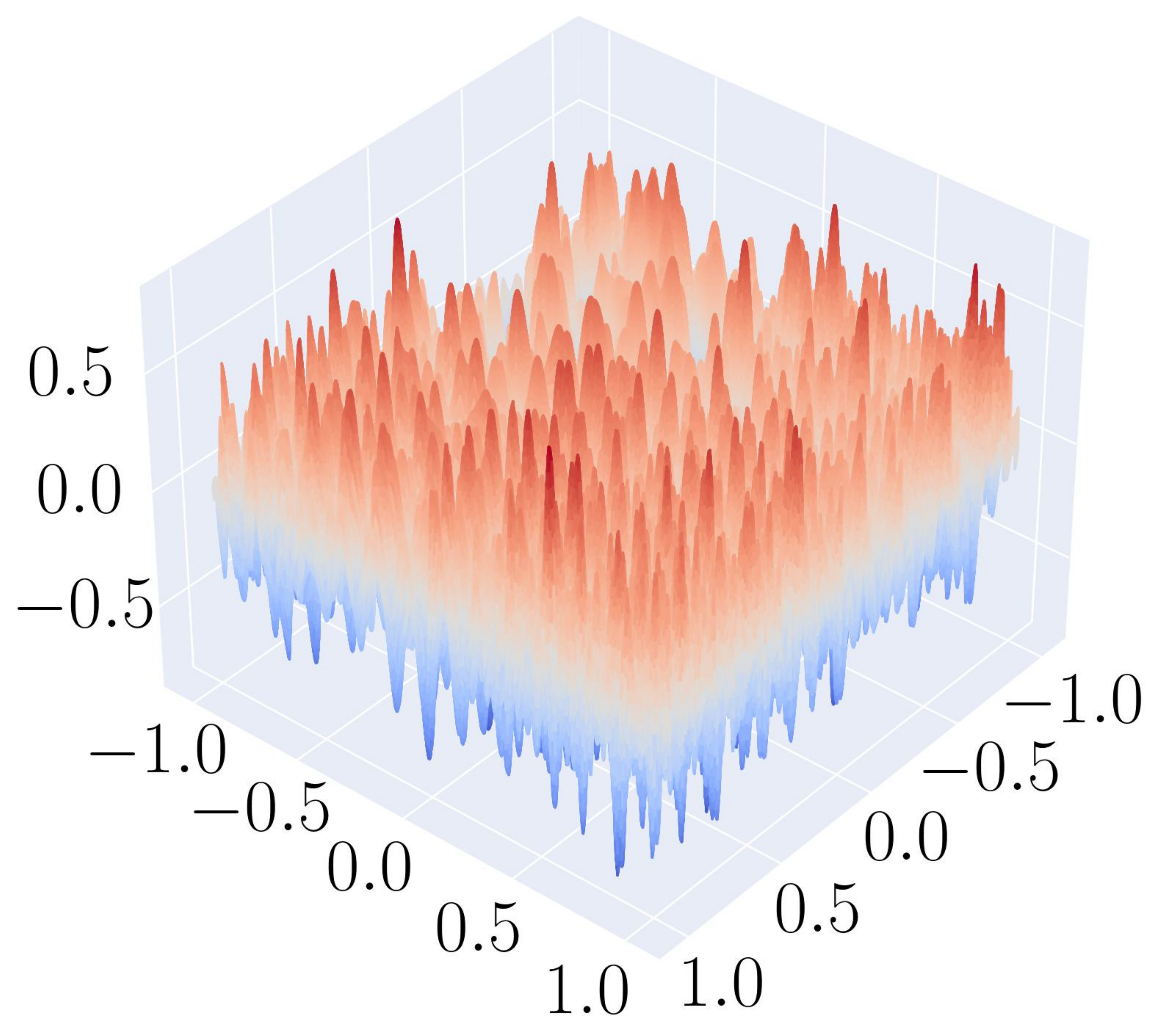}
    \end{subfigure}
    \caption{$f_2$.}
\end{subfigure}
\end{minipage}
\caption{The target functions $f_1$ and $f_2$ used in the numerical experiments, with $f_2$ shown as a heatmap and a surface plot.}
\label{fig:vsFCNN:f1D2D}
\end{figure}

\subsubsection*{Training protocol, data, and MGDL architectures}

All experiments are implemented in PyTorch, and all models are optimized using the Adam optimizer
\cite{DBLP:journals/corr/KingmaB14}.
Training is conducted in single precision by minimizing the mean-squared error
(MSE).
Network parameters are initialized using PyTorch’s default uniform initialization
$\mathcal{U}(-\sqrt{\kappa},\sqrt{\kappa})$ with $\kappa$ equal to the reciprocal of the input feature dimension, applied to both
weights and biases.
Each experiment is repeated over \(32\) independent random-seed trials.
For each trial, we first take the base-10 logarithm of each recorded error,
and then compute the mean and standard deviation across trials.
In Figures~\ref{fig:error:MGDL} and~\ref{fig:error:MGDL:vs:base},
the curves show the mean and the translucent envelopes indicate one
standard deviation. Table~\ref{tab:final:error:MGDL:vs:base} reports the
final errors as mean \(\pm\) standard deviation.

For the one-dimensional experiment, $10{,}000$ training samples are drawn
independently from the uniform distribution on $[-1,1]$ and the model is  trained with
mini-batch size $400$.
Test errors are evaluated using $3{,}000$ fresh samples drawn from the same
distribution.
For the two-dimensional experiment, training is performed on a uniform Cartesian
grid of size $500\times500$ over $[-1,1]^2$ with mini-batch size $1{,}000$.
Test errors are computed using $90{,}000$ points sampled independently and
uniformly from $[-1,1]^2$.

We now describe the MGDL architectures and their grade-wise training strategy,
with emphasis on how residual errors evolve across successive grades.

\begin{itemize}
    \item  
The one-dimensional MGDL model takes the form
\[
\sum_{\ell=1}^{4}
\mathcal{A}_{\ell}^{\mathrm{out}}
\circ
\ocomp_{i=1}^{\ell}
(\ReLU\circ\mathcal{A}_{2i}\circ\ReLU\circ\mathcal{A}_{2i-1}).
\]
Here, all affine maps, including the output maps, belong to
$\aff_{\le 200}$. In particular, each hidden layer has width $200$,
and each grade has a scalar-valued affine output map.
At grade~$\ell$, only the newly introduced affine maps
$\mathcal{A}_{2\ell-1}$, $\mathcal{A}_{2\ell}$, $\mathcal{A}_{\ell}^{\mathrm{out}}$ are optimized, while all previously trained
components are kept frozen.
A total of $12{,}000$ training epochs are distributed across grades as
$(750,1500,3000,6750)$.
At each grade, the learning rate follows the schedule
$\eta_k=\eta_0\,0.9^{\lfloor k/s\rfloor}$ with $\eta_0=0.001$ and $s=120$.

\item
The two-dimensional MGDL model follows the same multigrade paradigm and is given
by
\[
\sum_{\ell=1}^{3}
\mathcal{A}_{\ell}^{\mathrm{out}}
\circ
\ocomp_{i=1}^{\ell}
(\ReLU\circ\mathcal{A}_{3i}\circ\ReLU\circ\mathcal{A}_{3i-1}\circ\ReLU\circ\mathcal{A}_{3i-2}).
\]
Here, each affine map, including the output maps
\(\mathcal{A}_{\ell}^{\mathrm{out}}\), is chosen from \(\aff_{\le 200}\). Specifically,  all hidden layers have width $200$, and each grade has a scalar affine output map.
Training is performed for a total of $1{,}200$ epochs, allocated across grades as
$(150,300,750)$.
At each grade,  the learning rate at epoch $k$ is given by $\eta_k=\eta_0\,0.9^{\lfloor k/s\rfloor}$, where $\eta_0=0.001$ and $s=12$.
\end{itemize}

\subsubsection*{Grade-wise residual decay}

Figure~\ref{fig:error:MGDL} shows the evolution of training and test
errors, with shaded regions marking individual grades. The empirical
test $L^\infty$ error (MaxE) is defined as the maximum absolute error
over all test samples. In both examples, the curves exhibit a clear
grade-wise structure, with substantial error reductions during
training within newly introduced grades and an overall downward
trend across the reported grades. These observations illustrate
the practical effectiveness of grade-wise residual refinement.

\begin{figure}
\centering    
    \includegraphics[width=0.9985\textwidth]{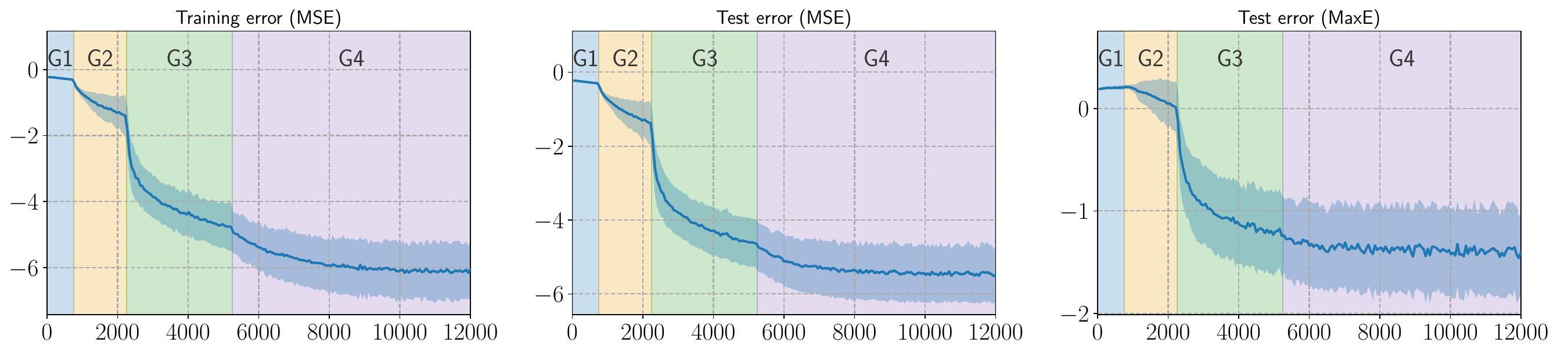}
    
   \vspace{4pt}
    \includegraphics[width=0.9985\textwidth]{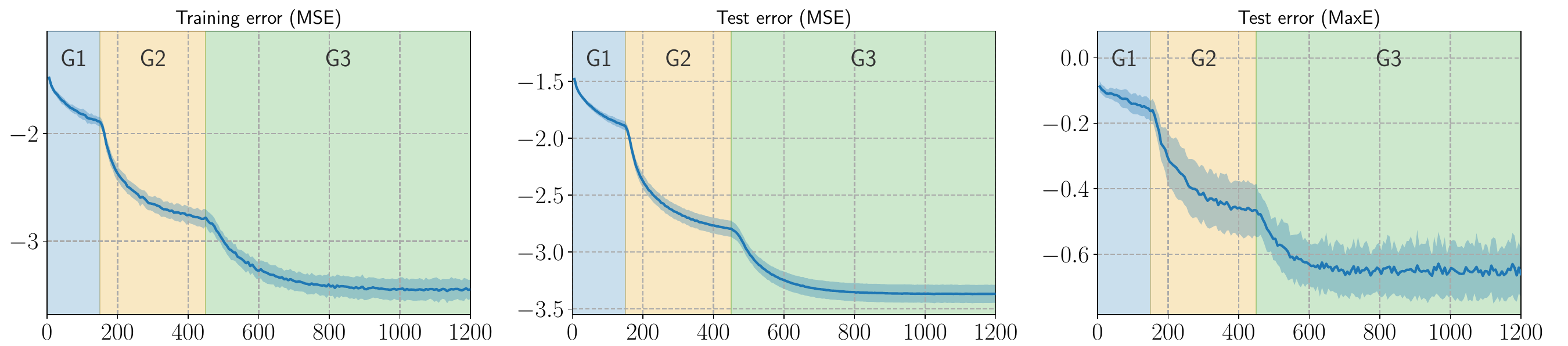}
\caption{Error curves for \(f_1\) (top row) and \(f_2\) (bottom row). 
    The curves and translucent envelopes denote the mean and one standard deviation over \(32\) independent trials, respectively. The shaded regions mark individual grades (\(\mathrm{G}1\)--\(\mathrm{G}4\) for \(f_1\) and \(\mathrm{G}1\)--\(\mathrm{G}3\) for \(f_2\)).
The horizontal axis denotes training epochs, and the vertical axis reports the base-10 logarithm of the error.}
    \label{fig:error:MGDL}
\end{figure}

\subsubsection*{Supplementary comparison with end-to-end training}

For completeness, we also report results obtained from standard fully connected
neural networks (FCNNs) trained end to end.
The FCNN baselines use the same activation functions, network width, overall depth, optimizer, data, batch size, and total number of
training epochs as the MGDL trunk. The learning-rate schedule has the same exponential-decay form, with step sizes chosen for the corresponding end-to-end training runs.
We note that MGDL introduces an output affine map for each grade, which leads to a slightly larger parameter count. The increase is negligible: 282,604 versus 282,001 in the one-dimensional experiment, and 322,803 versus 322,401 in the two-dimensional experiment, corresponding to about 0.21\% and 0.12\%, respectively.
More specifically, in the one-dimensional experiment, all FCNN parameters are
optimized simultaneously for $12{,}000$ epochs.
The learning rate at epoch $k$ follows the schedule
$\eta_k=\eta_0\,0.9^{\lfloor k/s\rfloor}$ with $\eta_0=0.001$ and step size
$s=200$.
In the two-dimensional experiment, all parameters are optimized simultaneously
for $1{,}200$ epochs, using the same learning-rate decay rule with
$\eta_0=0.001$ and step size $s=20$.

\begin{table}
\centering
\setlength{\tabcolsep}{1.185em}
\renewcommand{\arraystretch}{1.05}
\caption{Final error comparison between the end-to-end FCNN baseline  and the MGDL model.
All entries are reported as mean \(\pm\) standard deviation over
\(32\) independent trials, computed after applying the base-10 logarithm
to each error.
}
\label{tab:final:error:MGDL:vs:base}

\resizebox{0.985\linewidth}{!}{
\begin{tabular}{cc c c c c c}
\toprule
& \multicolumn{2}{c}{Training error (MSE)}
& \multicolumn{2}{c}{Test error (MSE)}
& \multicolumn{2}{c}{Test error (MaxE)} \\
\cmidrule(lr){2-3}\cmidrule(lr){4-5}\cmidrule(lr){6-7}
& FCNN & MGDL & FCNN & MGDL& FCNN & MGDL \\
\midrule
$f_1$ & $-4.35 \pm 0.49$ & $-6.05 \pm 0.88$ & $-4.28 \pm 0.47$ & $-5.48 \pm 0.74$ & $-1.10 \pm 0.31$ & $-1.40 \pm 0.42$ \\
$f_2$ & $-3.13 \pm 0.11$ & $-3.45 \pm 0.09$ & $-3.12 \pm 0.10$ & $-3.37 \pm 0.08$ & $-0.53 \pm 0.08$ & $-0.66 \pm 0.07$ \\
\bottomrule
\end{tabular}
}
\end{table}

Final training and test errors are summarized in
Table~\ref{tab:final:error:MGDL:vs:base}, while the comparison of the error curves
is shown in Figure~\ref{fig:error:MGDL:vs:base}.
Under these comparable experimental conditions, MGDL achieves lower training and test errors than the end-to-end FCNN baseline. These results support the practical effectiveness of the structured grade-wise residual refinement strategy, especially for the oscillatory examples considered here. 
We also remark that, for \(f_1\), the MGDL results show more noticeable
trial-to-trial variability than in the two-dimensional case. A systematic
study of optimal training settings is an interesting direction but lies
beyond the scope of this paper; the settings used here are chosen
empirically rather than through extensive tuning.

\begin{figure}
\centering    
    \includegraphics[width=0.9985\textwidth]{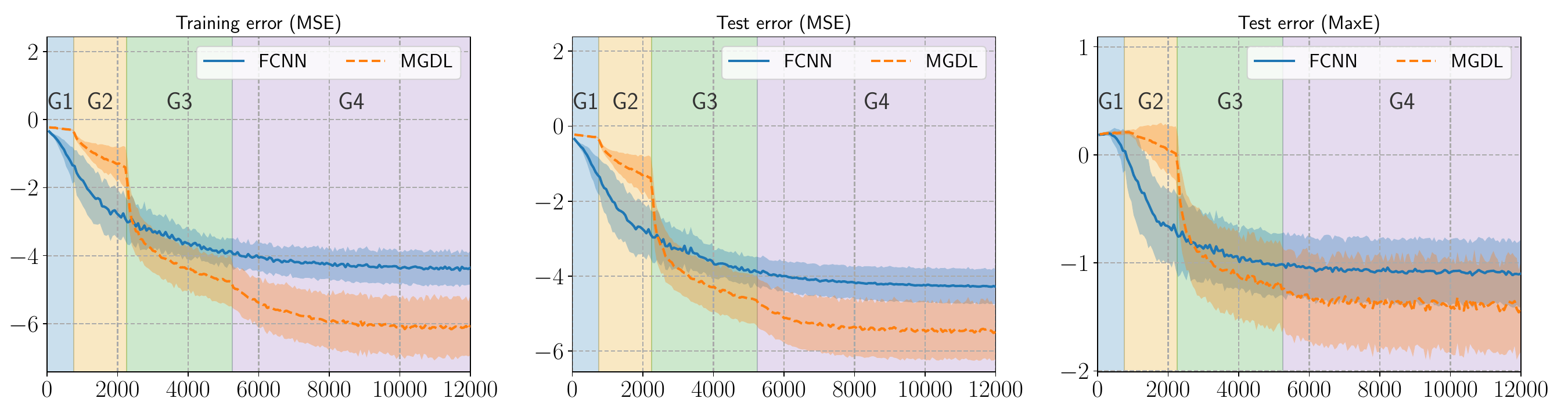}
     
     \vspace{4pt}
    \includegraphics[width=0.9985\textwidth]{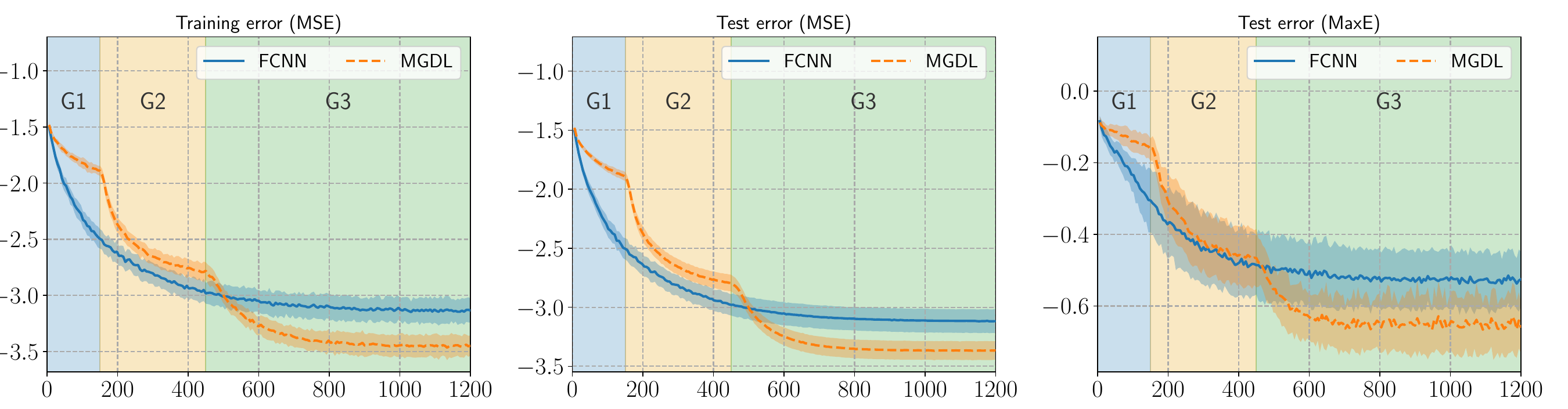}
\caption{Comparison of FCNN and MGDL via error curves for \(f_1\) (top row) and \(f_2\) (bottom row). 
    The curves show the mean over \(32\) independent trials, and the translucent envelopes around the curves indicate one standard deviation. 
    The background shaded regions indicate the MGDL grade schedule:
\(\mathrm{G}1\)--\(\mathrm{G}4\) for \(f_1\) and
\(\mathrm{G}1\)--\(\mathrm{G}3\) for \(f_2\).
The horizontal axis represents training epochs, and the vertical axis displays the base-\(10\) logarithm of the error.}
    \label{fig:error:MGDL:vs:base}
\end{figure}

\section{Proof of Theorem~\ref{thm:main:mgdl}}
\label{sec:proof:thm:main:mgdl}

This section proves the main approximation theorem,
Theorem~\ref{thm:main:mgdl}. The proof is organized around the one-step
contraction result stated as Theorem~\ref{thm:main:mgdl:one:op}. In
Section~\ref{sec:proof-strategy}, we first explain the reduction from the
global MGDL convergence theorem to this one-step contraction mechanism. In
Section~\ref{sec:detailed:proof:thm:main:mgdl}, we then iterate the contraction
construction and concatenate the corresponding affine maps to obtain the full
multigrade architecture.

\subsection{Proof strategy}
\label{sec:proof-strategy}

The proof of Theorem~\ref{thm:main:mgdl} proceeds in two  steps.  
First, for a prescribed $\varepsilon\in(0,(1+2^d)^{-1})$, we construct
a nonlinear operator $S:C([0,1]^d)\to C([0,1]^d)$ such that, for every
$g\in C([0,1]^d)$, the correction $Sg$ is realizable by a finite
\ReLU{} network whose parameters may depend on $g$, and
\[
    \|g-Sg\|_{L^\infty([0,1]^d)}
    \le (1-\varepsilon)\|g\|_{L^\infty([0,1]^d)}.
\]
The admissible range of $\varepsilon$ depends only on the dimension $d$. 
Second, we apply this contraction iteratively to the successive
MGDL residuals, obtaining geometric decay along a subsequence of
grades and thereby establishing Theorem~\ref{thm:main:mgdl}.
The first step constitutes the main technical contribution and is formalized in the following theorem.

\begin{theorem}\label{thm:main:mgdl:one:op}
Let $f \in C([0,1]^d)$ be a nonzero function and    let  $\varepsilon \in \bigl(0,\tfrac{1}{1+2^{d}}\bigr)$ be prescribed. There exist $n\in\N^+$ with $n\ge 2$ and affine maps 
\[
\mathcal{A}_{1}\in\aff(d,5d+1), \quad
\calA_2,\ldots,\calA_{n}\in\aff(5d+1,5d+1),\quad \calA_{n+1}\in\aff(5d+1,1),
\]
$\mathcal{A}_{2}^{\mathrm{out}},\ldots, \mathcal{A}_{n}^{\mathrm{out}}\in\aff(5d+1, 1)$,  $\calA_{n+1}^{\mathrm{out}}\in\aff(1, 1)$ 
such that, defining $R_1 \coloneqq f$ and
\[
R_{k+1} \coloneqq  R_{k} - \mathcal{A}_{k+1}^{\mathrm{out}} \circ 
   \ocomp_{i=1}^{k+1} (\ReLU \circ \mathcal{A}_i)\quad   \tn{for}\   k=1,\dots,n,
\]
the following properties hold:
\begin{enumerate}[label=(\roman*)]
\item 
\emph{Structural properties of the affine maps:}
$\mathcal{A}_2$  is independent of the  input coordinate  with index $5d+1$. Moreover, for   all  $\bmx\in[0,1]^d$,
\[
\left(
\ocomp_{i=1}^n (\ReLU\circ\mathcal{A}_i)(\bmx)
\right)[1:d]
=
\bmx .
\]

    \item \emph{Uniform accuracy:}  
    \[\|R_{n+1}\|_{L^\infty([0,1]^d)}\le (1-\eps)\,\|f\|_{L^\infty([0,1]^d)}.\]

    \item \emph{Pointwise domination:}  
    \[|R_{k+1}(\bmx)| \le |R_k(\bmx)|\quad \tn{for all}\   \bmx\in[0,1]^d\  \tn{and}  \ k=1,2,\ldots,n.\]

        \item \emph{Strict $L^p$ monotonicity for any \(p\in[1,\infty)\):}  
    \[\|R_{k+1}\|_{L^p([0,1]^d)} < \|R_k\|_{L^p([0,1]^d)}\quad   \tn{for}\   k=1,2,\dots,n.\]
\end{enumerate}
\end{theorem}

The proof of Theorem~\ref{thm:main:mgdl:one:op} is presented in Section~\ref{sec:proof:thm:main:mgdl:one:op}. 
For each fixed target $g$, the correction $Sg$ is realized by the
multigrade \ReLU{} construction of
Theorem~\ref{thm:main:mgdl:one:op}.
Iterative application of this contraction to the MGDL residuals yields the grade-wise error decay claimed in Theorem~\ref{thm:main:mgdl}.

\subsection{Main argument}
\label{sec:detailed:proof:thm:main:mgdl}

With Theorem~\ref{thm:main:mgdl:one:op} established, we prove
Theorem~\ref{thm:main:mgdl} by iteratively applying the one-step contraction
to the successive residuals. This iteration follows the MGDL refinement
mechanism: the correction operators constructed in
Theorem~\ref{thm:main:mgdl:one:op} are concatenated and used as the
grade-wise subnetworks. Before carrying out the construction, we record a
minor indexing convention. The first output map is chosen to be the zero map
so that the first nonzero correction appears at the next index. After
relabeling, this correction corresponds to the first effective MGDL grade, and
the convention has no effect on the practical grade-wise formulation in
Section~\ref{sec:mgdl-algorithm}.

\begin{proof}[Proof of Theorem~\ref{thm:main:mgdl}]
We assume $f \not\equiv 0$, as the case $f \equiv 0$ is trivial. 
Fix
$\varepsilon\in \big(0,\tfrac{1}{1+2^{d}}\big)$
throughout the proof.
Apply Theorem~\ref{thm:main:mgdl:one:op} to $R_1 \coloneqq f \in C([0,1]^d)$. 
There exist $n_1\in\N^+$ with $n_1\ge 2$ and affine maps 
\begin{gather*}
\mathcal{A}_{1}\in\aff(d,5d+1), \quad
\calA_2,\ldots,\calA_{n_1}\in\aff(5d+1,5d+1),\quad \tildecalA_{n_1+1}\in\aff(5d+1,1),
\\ 
\mathcal{A}_{2}^{\mathrm{out}},\ldots, \mathcal{A}_{n_1}^{\mathrm{out}}\in\aff(5d+1, 1),\quad   \tildecalA_{n_1+1}^{\mathrm{out}}\in\aff(1, 1)
\end{gather*}
such that, with 
\[
   R_{k+1} \coloneqq  R_{k} - \mathcal{A}_{k+1}^{\mathrm{out}} \circ    
   \ocomp_{i=1}^{k+1} (\ReLU \circ \mathcal{A}_i)\quad   \tn{for}\   k=1,\dots,n_1-1,
\]
and \[
   R_{n_1+1} \coloneqq  R_{n_1} - \caltildeA_{n_1+1}^{\mathrm{out}} \circ (\ReLU \circ \tildecalA_{n_1+1})\circ  
   \ocomp_{i=1}^{n_1} (\ReLU \circ \mathcal{A}_i),
\]
the following properties are satisfied:
\begin{enumerate}[label=(\roman*)]
\item 
\emph{Structural properties of the affine maps:}
$\mathcal{A}_2\in \aff(5d+1,5d+1)$  is independent of the  input coordinate  with index $5d+1$. Moreover, for 
all
$\bmx\in[0,1]^d$,
\begin{equation}
\label{eq:n1:id:x}
    \left(
\ocomp_{i=1}^{n_1}(\ReLU\circ\mathcal{A}_i)(\bmx)
\right)[1:d]
=
\bmx .
\end{equation}

    \item Uniform accuracy:
\begin{equation}
\label{eq:n1:Rn:unif:decay}
           \|R_{n_1+1}\|_{L^\infty([0,1]^d)}
       \le (1-\eps ) \|f\|_{L^\infty([0,1]^d)}.
\end{equation}

    \item Pointwise domination:
\begin{equation}
        \label{eq:n1:Rn:pointwise:decay}
       |R_{k+1}(\bmx)| \le |R_k(\bmx)|\quad   \tn{for all}\    \bmx\in[0,1]^d
       \tn{ and  } k=1,2,\dots,n_1.
\end{equation}

    \item Strict $L^p$ monotonicity for every $p\in[1,\infty)$:
\begin{equation}
\label{eq:n1:Rn:strict:Lp:decay}
           \|R_{k+1}\|_{L^p([0,1]^d)} < \|R_k\|_{L^p([0,1]^d)}\quad   \tn{for}\    k=1,2,\dots,n_1.
\end{equation}
\end{enumerate}
We remark that we will construct 
$\calA_{n_1+1}^{\mathrm{out}},  \calA_{n_1+1}$
based on $\caltildeA_{n_1+1}^{\mathrm{out}},  \tildecalA_{n_1+1}$.

Next, 
we assume $R_{n_1+1} \not\equiv 0$, since the case $R_{n_1+1} \equiv 0$ is trivial.
Applying Theorem~\ref{thm:main:mgdl:one:op} again to $R_{n_1+1}\in C([0,1]^d)$, 
we obtain an integer $n_2\in\N^+$ with $n_2\ge 2$ and affine maps 
\begin{gather*}
\hatcalA_{n_1+1}\in\aff(d,5d+1), \quad 
 \calA_{n_1+2},\ldots,\calA_{n_1+n_2}\in\aff(5d+1,5d+1),\quad 
 \tildecalA_{n_1+n_2+1}\in\aff(5d+1, 1),\\[3pt] 
\mathcal{A}_{n_1+2}^{\mathrm{out}},\ldots, \mathcal{A}_{n_1+n_2}^{\mathrm{out}}\in\aff(5d+1, 1),\quad   \tildecalA_{n_1+n_2+1}^{\mathrm{out}}\in\aff( 1, 1) 
\end{gather*}
such that, with 
\[
   R_{k+1} \coloneqq  R_{k} - \mathcal{A}_{k+1}^{\mathrm{out}} \circ    
   \left(\ocomp_{i=n_1+2}^{k+1} (\ReLU \circ \mathcal{A}_i)\right)\circ (\ReLU \circ \hatcalA_{n_1+1})  
\]
for $k=n_1+1,\dots,n_1+n_2-1$
and
\[
\begin{split}
	R_{n_1+n_2+1} \coloneqq  R_{n_1+n_2} - \tildecalA_{n_1+n_2+1}^{\mathrm{out}} 
	& \circ (\ReLU \circ \tildecalA_{n_1+n_2+1})
	 \circ  
	\left(\ocomp_{i=n_1+2}^{n_1+n_2} (\ReLU \circ \mathcal{A}_i)\right)\circ (\ReLU \circ \hatcalA_{n_1+1}),
\end{split}
\]
the following properties hold:
\begin{enumerate}[label=(\roman*)]
\item 
\emph{Structural properties of the affine maps:}
$\mathcal{A}_{n_1+2}\in \aff(5d+1,5d+1)$  is independent of the  input coordinate  with index $5d+1$. Moreover, for all
$\bmx\in[0,1]^d$,
\[
\left(\left(
\ocomp_{i=n_1+2}^{n_1+n_2}(\ReLU\circ\mathcal{A}_i)\right)\circ (\ReLU \circ \hatcalA_{n_1+1})(\bmx)
\right)[1:d]
=
\bmx .
\]

    \item Uniform accuracy:
    \[
       \|R_{n_1+n_2+1}\|_{L^\infty([0,1]^d)}
       \le (1-\eps ) \|R_{n_1+1}\|_{L^\infty([0,1]^d)}
       \le (1-\eps )^2 \|f\|_{L^\infty([0,1]^d)},
    \]
    where the second inequality comes from \eqref{eq:n1:Rn:unif:decay}.

    \item Pointwise domination:
    \[
       |R_{k+1}(\bmx)| \le |R_k(\bmx)|\quad   \tn{for all}\   \bmx\in[0,1]^d \tn{  and   } k=n_1+1,\dots,n_1+n_2.
    \]
    Combining this inequality with \eqref{eq:n1:Rn:pointwise:decay}, we have
        \[
       |R_{k+1}(\bmx)| \le |R_k(\bmx)|\quad   \tn{for all}\   \bmx\in[0,1]^d \tn{  and   } k=1,\dots,n_1+n_2.
    \]

        \item Strict $L^p$ monotonicity:
    \[
       \|R_{k+1}\|_{L^p([0,1]^d)} < \|R_k\|_{L^p([0,1]^d)}\quad   \tn{for}\   k=n_1+1,\dots,n_1+n_2.
    \]
    Combining this inequality with \eqref{eq:n1:Rn:strict:Lp:decay}, we have
        \[
       \|R_{k+1}\|_{L^p([0,1]^d)} < \|R_k\|_{L^p([0,1]^d)}\quad   \tn{for}\   k= 1,\dots,n_1+n_2.
    \]
\end{enumerate}

Define 
$\calA_{n_1+1}^{\mathrm{out}}\in \aff(5d+1,1)$
via 
\[
\calA_{n_1+1}^{\mathrm{out}}(\bmx,\bmy,z)=\tildecalA_{n_1+1}^{\mathrm{out}}(z)
\]
and
$\calA_{n_1+1}\in \aff(5d+1,5d+1)$ via
\begin{equation*}
    \calA_{n_1+1}\left(
    \begin{bmatrix}
        \bmx\\ \bmy\\ z
    \end{bmatrix}
    \right)
    =\begin{bmatrix}
        \big(\hatcalA_{n_1+1}(\bmx)\big)[1:5d]\\[4pt]
        \caltildeA_{n_1+1}(\bmx,\bmy,z)
    \end{bmatrix}
    \quad \tn{for all $\bmx\in\R^d,\, \bmy\in \R^{4d},\, z\in\R$}.
\end{equation*}
This definition retains the last coordinate needed to reproduce
the current correction, while using the first $5d$ coordinates
to prepare the features for the next contraction step.
The next affine map $\calA_{n_1+2}$ ignores the last coordinate.
It follows that
\[
\begin{split}
       R_{n_1+1}
       & = R_{n_1} - \caltildeA_{n_1+1}^{\mathrm{out}} \circ (\ReLU \circ \tildecalA_{n_1+1})\circ  
   \ocomp_{i=1}^{n_1} (\ReLU \circ \mathcal{A}_i)
   \\
   & = R_{n_1} - \calA_{n_1+1}^{\mathrm{out}} \circ (\ReLU \circ \calA_{n_1+1})\circ  
   \ocomp_{i=1}^{n_1} (\ReLU \circ \mathcal{A}_i)
   \\
   & = R_{n_1} - \calA_{n_1+1}^{\mathrm{out}} \circ  
   \ocomp_{i=1}^{n_1+1} (\ReLU \circ \mathcal{A}_i).
\end{split}
\]

Since $\mathcal{A}_{n_1+2}\in \aff(5d+1,5d+1)$ is independent of the input coordinate with index $5d+1$,
by \eqref{eq:n1:id:x}, for any $\bmx\in [0,1]^d$,
\[
\begin{split}
  \calA_{n_1+2}\circ (\ReLU \circ \hatcalA_{n_1+1})(\bmx)
 & =
\calA_{n_1+2}\circ (\ReLU \circ \hatcalA_{n_1+1})
\left(\left(\ocomp_{i=1}^{n_1}(\ReLU\circ\mathcal{A}_i)(\bmx)\right)[1:d]\right)
\\ &  =\calA_{n_1+2}\circ (\ReLU \circ \calA_{n_1+1})
\left( \ocomp_{i=1}^{n_1}(\ReLU\circ\mathcal{A}_i)(\bmx)\right)
\\ &   =\calA_{n_1+2}\circ 
\ocomp_{i=1}^{n_1+1}(\ReLU\circ\mathcal{A}_i)(\bmx)
.
\end{split}
\]
Then, for any $\bmx\in[0,1]^d$ and $k=n_1+1,\ldots,n_1+n_2-1$, we have
\begin{equation*}
    \begin{split}
           R_{k+1}(\bmx)
           & = R_{k}(\bmx) - \mathcal{A}_{k+1}^{\mathrm{out}} \circ    
   \left(\ocomp_{i=n_1+2}^{k+1} (\ReLU \circ \mathcal{A}_i)\right)\circ (\ReLU \circ \hatcalA_{n_1+1})(\bmx) 
 \\ &   = R_{k}(\bmx) - \mathcal{A}_{k+1}^{\mathrm{out}} \circ    
  \ocomp_{i=1}^{k+1} (\ReLU \circ \mathcal{A}_i)(\bmx).
    \end{split}
\end{equation*}
Therefore, for any $\bmx\in[0,1]^d$ and $k=1,\ldots,n_1+n_2-1$, we have
\begin{equation*}
    \begin{split}
           R_{k+1}(\bmx)
            = R_{k}(\bmx)  - \mathcal{A}_{k+1}^{\mathrm{out}} \circ    
  \ocomp_{i=1}^{k+1} (\ReLU \circ \mathcal{A}_i)(\bmx).
    \end{split}
\end{equation*}
In addition, for any $\bmx\in[0,1]^d$, we have
\[
\begin{split}
       R_{n_1+n_2+1}(\bmx)
  & = R_{n_1+n_2}(\bmx) - \tildecalA_{n_1+n_2+1}^{\mathrm{out}} \circ (\ReLU \circ \tildecalA_{n_1+n_2+1})\circ  
  \ocomp_{i=1}^{n_1+n_2} (\ReLU \circ \mathcal{A}_i)(\bmx)
  \\ & 
  =R_{n_1+n_2}(\bmx) - \calA_{n_1+n_2+1}^{\mathrm{out}} \circ 
  \ocomp_{i=1}^{n_1+n_2+1} (\ReLU \circ \mathcal{A}_i)(\bmx),
\end{split}
\]
where
$\mathcal{A}_{n_1+n_2+1}^{\mathrm{out}}$ and
$\mathcal{A}_{n_1+n_2+1}$ are constructed in the same way from
\[
\tildecalA_{n_1+n_2+1}^{\mathrm{out}},
\quad
\tildecalA_{n_1+n_2+1},
\quad
\hatcalA_{n_1+n_2+1}.
\]
Here,  $\hatcalA_{n_1+n_2+1}$ is obtained by applying
Theorem~\ref{thm:main:mgdl:one:op} to $R_{n_1+n_2+1}$.

We continue this construction recursively, applying the same argument
at each stage to the next residual in the sequence
$R_{n_1+n_2+1}$, $R_{n_1+n_2+n_3+1}$, and so on.
Let $k_j=1+\sum_{i=1}^{j-1}n_i$ denote the successive contraction
indices. If the residual does not vanish after finitely many grades,
then $k_j\to\infty$, and the construction gives
\[
\|R_{k_j}\|_{L^\infty([0,1]^d)}
\le (1-\varepsilon)^{j-1}\|f\|_{L^\infty([0,1]^d)}.
\]
Since the residual magnitudes are nonincreasing at every
intermediate grade, the entire residual sequence converges
uniformly to zero. If a residual vanishes, all subsequent
corrections may be chosen to be zero, and the same conclusion holds.
Thus,  affine maps 
\[
\mathcal{A}_{i}\in \aff_{\le 5d+1}
\quad \text{for } i=1,2,\ldots
\quad \tn{and}\quad
\mathcal{A}_{k}^{\mathrm{out}}\in \aff_{\le 5d+1}
\quad \text{for } k=2,3,\ldots 
\]
can be obtained such that, with
\[
   R_{k+1} = R_k - \mathcal{A}_{k+1}^{\mathrm{out}} \circ 
   \ocomp_{i=1}^{k+1} (\ReLU \circ \mathcal{A}_i)\quad \tn{on  } [0,1]^d  \quad \tn{for}\   k=1,2,\dots,
\]
the following properties are satisfied:
\begin{enumerate}[label=(\roman*)]
    \item Uniform accuracy:
    \[
       \|R_k\|_{L^\infty([0,1]^d)} \to 0
       \quad \text{as } k\to\infty.
    \]

        \item Pointwise domination:
    \[
       |R_{k+1}(\bmx)| \le |R_k(\bmx)|\quad   \tn{for all}\   \bmx\in[0,1]^d \tn{  and  } k=1,2,\dots.
    \]
    
    \item Strict $L^p$ monotonicity for every $p\in[1,\infty)$:
    \[
       \|R_{k+1}\|_{L^p([0,1]^d)} < \|R_k\|_{L^p([0,1]^d)}\quad \tn{if  }     \|R_k\|_{L^p([0,1]^d)}>0  \quad \tn{for}\   k=1,2,\dots.
    \]
\end{enumerate}
It follows from these properties that, for every $\bmx\in[0,1]^d$,
\[
    |R_k(\bmx)| \searrow 0
\quad \tn{and}\quad
    \|R_k\|_{L^\infty([0,1]^d)} \searrow 0 \quad \text{as } k\to\infty .
\]
Moreover, for every $p\in[1,\infty)$,
\[
    \|R_k\|_{L^p([0,1]^d)} \searrow 0
    \quad \text{as } k\to\infty,
\]
strictly until the residual becomes identically zero. More precisely, for each
$k$, either
\[
    R_{k+1}\equiv R_k\equiv 0
    \quad \text{on } [0,1]^d 
\quad \tn{or}\quad 
    \|R_{k+1}\|_{L^p([0,1]^d)}
    <
    \|R_k\|_{L^p([0,1]^d)} .
\]

Finally, 
we define the multigrade approximants
\[
\Phi_k\coloneqq
\sum_{\ell=1}^k
\mathcal{A}_{\ell}^{\mathrm{out}}
\circ
\ocomp_{i=1}^\ell (\ReLU \circ \mathcal{A}_i)\quad   \tn{for all}\   k\in\mathbb{N}^+,
\]
where $\mathcal{A}_1^{\mathrm{out}}\in\aff(5d+1,1)$ is chosen to be the zero affine map. 
This is only an indexing convention used to fit the two-hidden-layer cutoff realization
into the residual-sum form; the first nonzero correction appears at the next
index and the multigrade approximation property is unchanged.

Thus $\Phi_1\equiv 0$ and $R_1=f=f-\Phi_1$. Moreover, for every $k=1,2,\ldots$,
\[
\begin{split}
    R_{k+1}
    &=R_k-\mathcal{A}_{k+1}^{\mathrm{out}} \circ 
   \ocomp_{i=1}^{k+1} (\ReLU \circ \mathcal{A}_i)
   =\cdots
   =R_1- \sum_{\ell=2}^{k+1}\mathcal{A}_{\ell}^{\mathrm{out}} \circ 
   \ocomp_{i=1}^{\ell} (\ReLU \circ \mathcal{A}_i)
   \\  &  =f- \sum_{\ell=1}^{k+1}\mathcal{A}_{\ell}^{\mathrm{out}} \circ 
   \ocomp_{i=1}^{\ell} (\ReLU \circ \mathcal{A}_i).
\end{split}
\]
Therefore,
\[
\begin{split}
    R_{k}
     =f- \sum_{\ell=1}^{k}\mathcal{A}_{\ell}^{\mathrm{out}} \circ 
   \ocomp_{i=1}^{\ell} (\ReLU \circ \mathcal{A}_i)\quad 
   \tn{for  } k=1,2,\ldots.
\end{split}
\]
We also note that the construction may be terminated once the residual vanishes.
Indeed, if \(R_{k_0}\equiv 0\) for some \(k_0\), then one may choose
\(\calA_k^{\tn{out}}\) to be the zero affine map for all \(k\ge k_0+1\).
Consequently, \(\Phi_k=\Phi_{k_0}\) for all \(k\ge k_0+1\), and hence
\(R_k\equiv 0\) for all \(k\ge k_0\).
This completes the proof.
\end{proof}

We remark that, in Corollary~\ref{coro:mgdl:residual}, the subsequence
$\{k_j\}_{j=1}^\infty$ can be constructed in the form
\[
k_j = 1+\sum_{i=1}^{j-1} n_i,
\]
where each $n_i$ depends on the prescribed $\varepsilon$ and the
current residual $R_{k_i}$. The construction for nonzero residuals
will be given in the proof of Theorem~\ref{thm:main:mgdl:one:op}.
If $R_{k_j}\equiv0$ for some $j$, we set $n_i=1$ for all $i\ge j$,
so that $k_{j+\ell}=k_j+\ell$ for all $\ell\ge1$;
the subsequent contraction inequalities then hold trivially.

\section{Proof of Theorem~\ref{thm:main:mgdl:one:op}}
\label{sec:proof:thm:main:mgdl:one:op}

This section establishes the one-step contraction theorem,
Theorem~\ref{thm:main:mgdl:one:op}, which is the technical core of the paper.
Section~\ref{sec:ideas:proof:thm:main:mgdl:one:op} gives the high-level proof
strategy: localization of the near-maximum region, construction of cutoff
functions, and assembly of the contraction operator. Section~\ref{sec:preliminaries:proof:thm:main:mgdl:one:op}
develops the required contraction notions, geometric covering estimates, and
cutoff-function preliminaries. 
Finally, Section~\ref{sec:detailed:proof:thm:main:mgdl:one:op}
realizes the correction for each fixed target using a finite
multigrade \ReLU{} architecture and proves the desired uniform
contraction, pointwise domination, and strict $L^p$ monotonicity.

\subsection{Proof strategy}
\label{sec:ideas:proof:thm:main:mgdl:one:op}

We  summarize here the proof strategy for Theorem~\ref{thm:main:mgdl:one:op} and outline the construction that leads
to the desired one-step contraction.
    Our construction relies on separating the positive and negative contributions of a function and controlling them via suitable operators. For this reason, we begin by fixing notation for the positive and negative parts of a function.

\begin{definition}
For any real-valued function $g$, we write $g = g^{+} - g^{-}$, where
\begin{equation*}
    g^{+}(\bmx)=\max\{g(\bmx),0\}\quad \tn{and}\quad 
    g^{-}(\bmx)=\max\{-g(\bmx),0\}.
\end{equation*}
The functions $g^{+}$ and $g^{-}$ denote the positive and negative parts of $g$, and this notation will be used throughout the paper.
\end{definition}

It is enough to describe the construction for the case \(g(\bmx)\ge 0\) on \([0,1]^d\). Indeed, for a general target function \(g\), we decompose \(g=g^+-g^-\), apply the same construction separately to \(g^+\) and \(g^-\), and then combine the two resulting corrections with opposite signs. Let \(\varepsilon\in(0,1)\) be a small parameter to be chosen later. The argument proceeds through the following three conceptual steps.

\begin{enumerate}[label=\emph{Step \arabic*:}, leftmargin=4.5em]

\item \emph{Localization of the near-maximum region.}  

We first identify the subset on which \(g\) is close to its supremum, namely,
\[
    E_{g,\eps}
     \coloneqq
    \bigl\{\bmx\in[0,1]^d:
    g(\bmx) >  (1-\eps)\|g\|_{L^\infty([0,1]^d)}\bigr\}.
    \]
Figure~\ref{fig:cutoff:functions} illustrates the region $E_{g,\eps}$.
This set is then covered by finitely many essentially disjoint hypercubes, which allows the near-maximum region to be localized in geometrically simple sets.

\item \emph{Construction of cutoff functions.}  

For each hypercube \(Q\) in the above cover, we can construct a continuous cutoff function that

(i) attains the constant value \(\eps \|g\|_{L^\infty([0,1]^d)}\)  on \(Q\), 

(ii) decays continuously to zero on a slightly enlarged cube, and  

(iii) vanishes identically outside this enlargement. 

Each such cutoff function can be implemented exactly by a simple \ReLU{} network.
Figure~\ref{fig:cutoff:functions} schematically illustrates this construction. 

\item \emph{Assembly of the contraction operator.}  

We define  \(Tg\) as the sum of all cutoff functions, with the amplitudes scaled so that \(Tg\) 
satisfies the pointwise bound
\[
    |Tg(\bmx)| \le |g(\bmx)| \quad  \text{for all }\bmx.
\]
To guarantee this inequality, we carefully account for potential overlaps among the supports of the cutoff functions as illustrated in Figure~\ref{fig:cutoff:functions} and choose the contraction parameter $\varepsilon$ sufficiently small so that the aggregate contribution of all cutoffs remains dominated by \(g\).  
Under our construction, the number of overlapping cutoff supports at any point is uniformly bounded by \(2^d\), as  established in Lemma~\ref{lem:bounded:overlap:scaling}. Consequently, it suffices to impose the condition
\[
\varepsilon < \frac{1}{1+2^d},
\]
which will be used repeatedly in the detailed proof. 
With such a choice of \(\varepsilon\),
we obtain the uniform estimate
\[
    \|g - Tg\|_{L^\infty([0,1]^d)} 
    \le (1-\varepsilon)\|g\|_{L^\infty([0,1]^d)} ,
\]
yielding a uniform contraction of the residual. The remaining conclusions of Theorem~\ref{thm:main:mgdl:one:op}, including strict \(L^p\) monotonicity and pointwise domination, emerge naturally from the preceding construction.
\end{enumerate}

\begin{figure}
\centering
\includegraphics[width=0.84\textwidth]{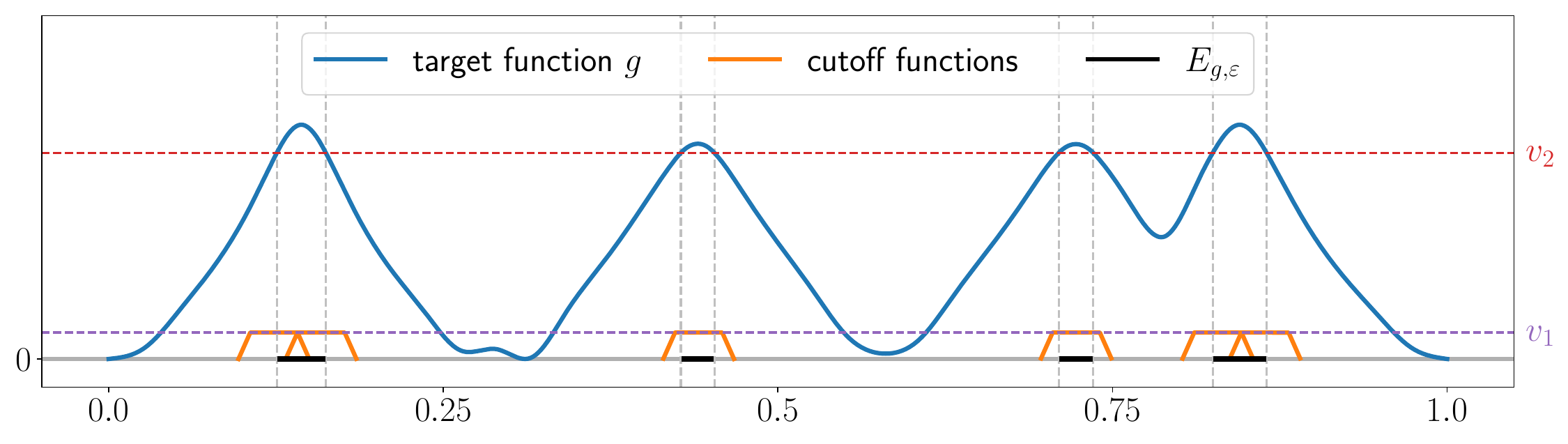}
\caption{An illustration of  cutoff functions associated with a target function $g$, where $v_1=\eps \|g\|_{L^\infty([0,1]^d)}$ and $v_2= (1-\eps) \|g\|_{L^\infty([0,1]^d)}$.}
\label{fig:cutoff:functions}
\end{figure}

\subsection{Preliminaries}
\label{sec:preliminaries:proof:thm:main:mgdl:one:op}

We now present the technical preliminaries underlying the proof of
Theorem~\ref{thm:main:mgdl:one:op}.
As outlined in Section~\ref{sec:ideas:proof:thm:main:mgdl:one:op}, the proof
follows a three-step strategy centered on constructing a suitable
contraction mechanism.
To formalize this approach, we first introduce the basic notions and
structural assumptions on operators that will be used throughout the
analysis, and establish several auxiliary definitions and lemmas that
support the subsequent construction.

Throughout this section, we fix a compact domain $\Omega \subset \mathbb{R}^n$ and a
small parameter $\varepsilon \in (0,1)$.
All definitions and results below are formulated with respect to these
standing assumptions.
We first introduce two notions of contractivity, which quantify the extent to which an operator drives a function toward its positive or balanced components.

\begin{definition}
An operator $T:C(\Omega)\to C(\Omega)$ is called a \emph{positive $\varepsilon$-contraction} if for any $g\in C(\Omega)$,
\[
    0\le g^{+}(\bm{x})-Tg(\bm{x})
    \le (1-\varepsilon)\|g^{+}\|_{L^\infty(\Omega)}\quad \tn{for all }\bm{x}\in\Omega.
\]
Similarly, $T$ is called a \emph{balanced $\varepsilon$-contraction} if for any $g\in C(\Omega)$,
\[
    \|g -Tg\|_{L^\infty(\Omega)}
    \le (1-\varepsilon)\|g \|_{L^\infty(\Omega)}.
\]
\end{definition}
The positive $\varepsilon$-contraction controls only the positive part $g^{+}$, while the balanced $\varepsilon$-contraction acts on the full function $g$. In the proof of Theorem~\ref{thm:main:mgdl:one:op}, we will first construct a positive contraction and then convert it into a balanced one by symmetrization.

To verify these contraction properties for a given operator, it is useful to isolate two basic structural conditions. 
The first provides a pointwise upper bound on $Tg$, expressed in terms of the positive part $g^{+}$. 
The second asserts that, whenever $g$ is sufficiently close to the maximum of $g^{+}$, the operator $T$ preserves a fixed proportion of this maximal value.

\begin{enumerate}[label=(C\arabic*)]
\item \label{item:T:C1}  {Pointwise bounds:} for all $g\in C(\Omega)$ and $\bm{x}\in\Omega$,
\[
0 \le Tg(\bm{x}) \le g^{+}(\bm{x}).
\]

\item \label{item:T:C2}  {Stability near the maximum:} for all $g\in C(\Omega)$ and $\bm{x}\in\Omega$,
\[
g(\bm{x}) >  (1-\varepsilon)\|g^{+}\|_{L^\infty(\Omega)}
\quad\Longrightarrow\quad
Tg(\bm{x}) \ge \varepsilon\|g^{+}\|_{L^\infty(\Omega)}.
\]
\end{enumerate}

Intuitively, condition~\ref{item:T:C1} enforces that $T$ never overshoots the positive part of $g$, while condition~\ref{item:T:C2} guarantees that $T$ retains a fixed fraction of the peak value whenever $g$ is close to its maximum. Together, they are precisely tailored to yield a positive contraction in the sense defined above.
The next lemma shows that these two simple structural conditions are already sufficient to imply a quantitative one-sided contraction.

\begin{lemma}\label{lem:P1P2:to:positive:eps:constraction}
Suppose that an operator $T:C(\Omega)\to C(\Omega)$ satisfies conditions \ref{item:T:C1} and \ref{item:T:C2}.  
Then $T$ is a positive $\varepsilon$-contraction. 
\end{lemma}
\begin{proof}
Fix $g\in C(\Omega)$ and define the level set
\[
    E_{g,\varepsilon}
    \coloneqq    \big\{\bm{x}\in\Omega : g(\bm{x}) >  (1-\varepsilon)\|g^{+}\|_{L^\infty(\Omega)}\big\}.
\]
We consider separately the cases $\bm{x}\in\Omega\setminus E_{g,\varepsilon}$ and $\bm{x}\in E_{g,\varepsilon}$.

\medskip
\noindent\textbf{Case 1: $\bm{x}\in\Omega\setminus E_{g,\varepsilon}$.}

By definition of $E_{g,\varepsilon}$ we have
\[
    g(\bm{x}) \le (1-\varepsilon)\|g^{+}\|_{L^\infty(\Omega)}.
\]
In particular,
\[
    g^{+}(\bm{x})=\max\{g(\bm{x}),0\}
    \le (1-\varepsilon)\|g^{+}\|_{L^\infty(\Omega)}.
\]
Using the pointwise bound \ref{item:T:C1}, we obtain
\[
    0\le Tg(\bm{x})\le g^{+}(\bm{x}),
\]
and hence
\[
    0\le g^{+}(\bm{x})-Tg(\bm{x})
    \le g^{+}(\bm{x})
    \le (1-\varepsilon)\|g^{+}\|_{L^\infty(\Omega)}.
\]

\medskip
\noindent\textbf{Case 2: $\bm{x}\in E_{g,\varepsilon}$.}

In this case, the stability condition \ref{item:T:C2} yields
\[
    Tg(\bm{x})\ge\varepsilon\|g^{+}\|_{L^\infty(\Omega)}.
\]
On the other hand, by \ref{item:T:C1} we have $Tg(\bm{x})\le g^{+}(\bm{x})\le\|g^{+}\|_{L^\infty(\Omega)}$.  
Combining these two bounds, we obtain
\[
    0\le g^{+}(\bm{x})-Tg(\bm{x})
    \le \|g^{+}\|_{L^\infty(\Omega)}-\varepsilon\|g^{+}\|_{L^\infty(\Omega)}
    = (1-\varepsilon)\|g^{+}\|_{L^\infty(\Omega)}.
\]

Combining the two cases, we conclude that for every $\bm{x}\in\Omega$,
\[
    0\le g^{+}(\bm{x})-Tg(\bm{x})
    \le (1-\varepsilon)\|g^{+}\|_{L^\infty(\Omega)}.
\]
This is precisely the positive $\varepsilon$-contraction estimate for $T$ on the positive part of $g$.
\end{proof}


Having established a criterion for positive contractions, we now show how to extend such one-sided control to a fully balanced contraction via a standard symmetrization argument.

\begin{lemma}\label{lem:positive:to:balanced:eps:contraction}
If $T$ is a positive $\varepsilon$-contraction, then the
symmetrized operator $S$, defined by
\[
Sg \coloneqq Tg-T(-g),
\]
is a balanced $\varepsilon$-contraction.
\end{lemma}

\begin{proof}
Fix $g\in C(\Omega)$ and decompose $g=g^{+}-g^{-}$ with $g^{+},g^{-}\ge0$.  Note that
\[
    (-g)^{+}(\bm{x})
    = \max\{-g(\bm{x}),0\}
    = g^{-}(\bm{x}).
\]
Applying the positive $\varepsilon$-contraction property to $g$ and $-g$ gives, for all $\bm{x}\in\Omega$,
\[
    0\le g^{+}(\bm{x})-Tg(\bm{x})
    \le (1-\varepsilon)\|g^{+}\|_{L^\infty(\Omega)}\le (1-\varepsilon)\|g\|_{L^\infty(\Omega)},
\]
and
\[
    0\le g^{-}(\bm{x})-T(-g)(\bm{x})
    \le (1-\varepsilon)\|g^{-}\|_{L^\infty(\Omega)}\le (1-\varepsilon)\|g\|_{L^\infty(\Omega)}.
\]

By defining
\[
    a(\bm{x}) \coloneqq    g^{+}(\bm{x})-Tg(\bm{x}) \quad \tn{and}\quad  
    b(\bm{x}) \coloneqq    g^{-}(\bm{x})-T(-g)(\bm{x}),
\]
we have
\[
\begin{aligned}
    g(\bm{x}) - Sg(\bm{x})
    &= \big(g^+(\bm{x}) - g^-(\bm{x})\big) - \big(Tg(\bm{x}) - T(-g)(\bm{x})\big) \\
    &= \big(g^{+}(\bm{x}) - Tg(\bm{x})\big)
       - \big(g^{-}(\bm{x}) - T(-g)(\bm{x})\big) 
       = a(\bm{x}) - b(\bm{x}).
\end{aligned}
\]
Since both $a(\bm{x})$ and $b(\bm{x})$ lie between $0$ and $(1-\varepsilon)\|g\|_{L^\infty(\Omega)}$, their difference satisfies
\[
-(1-\varepsilon)\|g\|_{L^\infty(\Omega)}
\le a(\bm{x}) - b(\bm{x})
\le (1-\varepsilon)\|g\|_{L^\infty(\Omega)}.
\]

Hence, for all $\bm{x}\in\Omega$,
\[
    |g(\bm{x}) - Sg(\bm{x})|
    \le (1-\varepsilon)\|g\|_{L^\infty(\Omega)},
\]
which yields
\[
    \|g - Sg\|_{L^\infty(\Omega)}
    \le (1-\varepsilon)\|g\|_{L^\infty(\Omega)}.
\]
Thus $S$ is a balanced $\varepsilon$-contraction.
\end{proof}

Having established the structural framework and the overall contraction mechanism, we now turn to the concrete operator constructions used in our main theorem.
To construct an operator satisfying conditions~\ref{item:T:C1} and \ref{item:T:C2}, we rely on a geometric decomposition of the domain into small hypercubes and associate to each cube a localized cutoff function. These cutoffs will later be implemented by two-hidden-layer \ReLU{} subnetworks.

\begin{definition}[Dilated hypercubes and continuous cutoff functions]
\label{def:dilated-cube-cutoff}
Let $Q \subset \R^d$ be an axis-aligned cube with center $\bmc(Q)$ and side length $\ell(Q)$.  
For any $r>1$, the \emph{$r$-dilate} of $Q$ is defined by
\[
  rQ \coloneqq \bigl\{\, \bmc(Q) + r(\bmx - \bmc(Q)) : \bmx \in Q \,\bigr\},
\]
so that $rQ$ is the cube concentric with $Q$ and of side length $r\,\ell(Q)$.

A \emph{(continuous) cutoff function} associated with $Q$ is any  continuous map $\Gamma_Q : \R^d \to [0,1]$ satisfying $\Gamma_Q(\bmx)=1$ for $\bmx\in Q$,
$0\le \Gamma_Q(\bmx)\le 1$ for $\bmx\in rQ\setminus Q$,
and $\Gamma_Q(\bmx)=0$ for $\bmx\notin rQ$.
Throughout this paper, we fix a dilation parameter $r\in(1,2)$, for example
$r=\frac{3}{2}$.  
To facilitate the implementation of cutoff functions using \ReLU{} networks,
we 
introduce the one-dimensional
function
\[
\psi(x)
\coloneqq
\frac{1}{r-1}\Bigl(\ReLU(x+r)-\ReLU(x+1)
+\ReLU(-x+r)-\ReLU(-x+1)\Bigr)-1.
\]
Figure~\ref{fig:psi:TrapezoidFunc} illustrates the function $\psi$ for
$r=\frac{3}{2}$.

\begin{figure}
\centering
\includegraphics[width=0.59\textwidth]{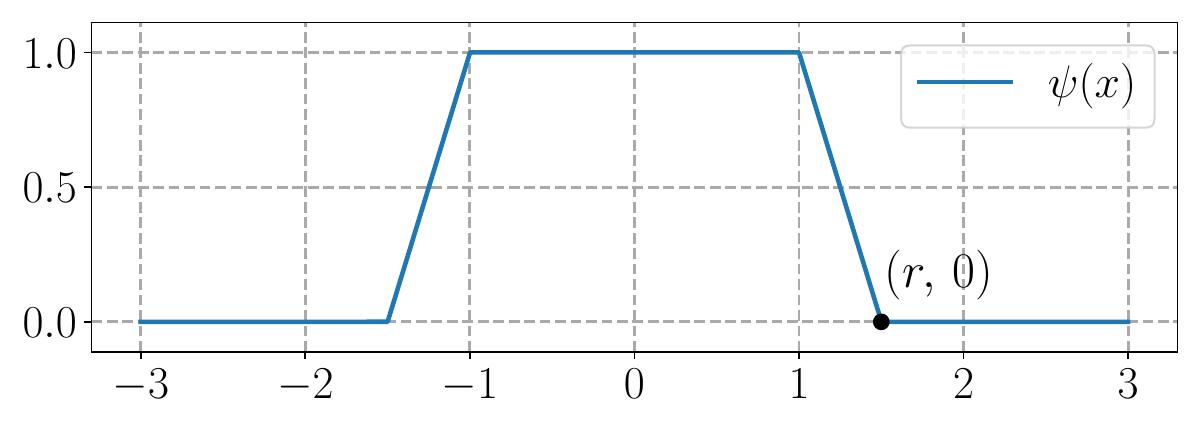}
\caption{An illustration of $\psi$.}
\label{fig:psi:TrapezoidFunc}
\end{figure}

Based on $\psi$, we define the $d$-dimensional function
\[
\Psi=\ReLU\circ \tildePsi,\quad \tn{where }
\tildePsi(\bmx)
\coloneqq
1-d+ \sum_{i=1}^d \psi(x_i).
\]
Since each coordinate function $\psi(x_i)$ is a linear combination
of four ReLU functions plus a constant, the map $\Psi$ can be realized by a two-hidden-layer \ReLU\ network:
there exist affine maps $\calA_1\in \aff(d,4d)$ and $\calA_2\in \aff(4d,1)$ such that
\[
\Psi=\ReLU\circ \calA_2\circ \ReLU\circ \calA_1.
\]
The cutoff function associated with $Q$ can be defined by the rescaled map
\[
\Gamma_Q(\bmx)
\coloneqq
\Psi\!\left(\tfrac{2}{\ell(Q)}\bigl(\bmx-\bmc(Q)\bigr)\right),
\]
where $\ell(Q)$ and $\bmc(Q)$ denote the side length and center of $Q$,
respectively. 
Equivalently, define the affine rescaling map
\[
\widetilde{\calA}_Q(\bmx)
\coloneqq
\tfrac{2}{\ell(Q)}\bigl(\bmx-\bmc(Q)\bigr).
\]
Then the cutoff function can be written as
\[
\Gamma_Q
=
\ReLU\circ \calA_2\circ \ReLU\circ \calA_1\circ \widetilde{\calA}_Q=\ReLU\circ \calA_2\circ \ReLU\circ \tildecalA_1,
\]
where $\tildecalA_1\coloneqq \calA_1\circ \widetilde{\calA}_Q$ is still an affine map in $\aff(d,4d)$.
It is straightforward to verify that \(\Psi=1\) on \([-1,1]^d\),
\(0\le \Psi\le1\) on \([-r,r]^d\), and \(\Psi=0\) outside
\([-r,r]^d\). Consequently, \(\Gamma_Q=1\) on \(Q\), \(0\le\Gamma_Q\le1\)
on \(rQ\), and \(\Gamma_Q=0\) outside \(rQ\).

\end{definition}

\medskip
Below we present an explicit construction of an operator $T_\epsilon$ on $C([0,1]^d)$ that satisfies conditions~\ref{item:T:C1} and \ref{item:T:C2} and therefore defines a positive $\varepsilon$-contraction. The construction is based on the dilated cube coverings and cutoff functions introduced above.
Given $\eps\in \big(0,\tfrac{1}{1+2^d}\big)$, we 
define the operator $T_\eps:C([0,1]^d)\to C([0,1]^d)$ as follows.  
For any $g\in C([0,1]^d)$, set
\[
   M_g \coloneqq  \|g^{+}\|_{L^\infty([0,1]^d)}.
\]
If $M_g=0$, we set $n_{g,\eps}:=0$ and define $T_\eps g\equiv0$.
In the nontrivial case $M_g>0$, the construction proceeds in four steps below.

\begin{enumerate}[label={(\roman*)}]

\item Introduce the superlevel set
\[
   E_{g,\eps} \coloneqq \{\bmx \in [0,1]^d : g(\bmx) 
   > (1-\eps) M_g\}.
\]

\item Define
\[
   \delta_{g,\eps} \coloneqq \sup\Bigl\{\, t \in (0,1) : 
      \omega_{g^+}(2\sqrt{d}\,t) \le (1-\eps-2^d\eps)M_g \Bigr\},
\]
where the positive part \(g^{+}\) is used to ensure uniform control over regions in which \(g \ge 0\).
Here, for a general function \(f \in C([0,1]^d)\), the modulus of continuity \(\omega_f\) is defined by
\[
\omega_f(t)
\coloneqq
\sup\Bigl\{
|f(\bmx)-f(\bmy)| :
\|\bmx-\bmy\|_2 \le t,\;
\bmx,\bmy \in [0,1]^d
\Bigr\} 
\quad \tn{for all }  t \ge 0.
\]

We note that \(\varepsilon\) is chosen sufficiently small so that the factor
\(1-\varepsilon-2^{d}\varepsilon\) remains positive.
Since $g^+\in C([0,1]^d)$ and $\omega_{g^+}(t)\to 0$ as $t\to 0$,
the admissible set in the definition of $\delta_{g,\epsilon}$ is nonempty, and hence $\delta_{g,\epsilon}>0$.

\item 
    We then choose a smaller number
\(\tildedelta_{g,\eps}\in(0,\delta_{g,\eps})\), to be specified below, and
consider the uniform grid of cubes
\[
   \widetilde Q_{g,\eps,\bm{\beta}}
   = \prod_{j=1}^d [\beta_j \tildedelta_{g,\eps},\, (\beta_j+1)\tildedelta_{g,\eps}]\quad \tn{for all }  \bm{\beta} \in \mathbb{Z}^d,
\]
which form a cover for $\R^d$.
Since $E_{g,\eps} \subseteq [0,1]^d$ is bounded, only finitely many such cubes intersect it.  
Let $\{Q_{g,\eps,i}\}_{i=1}^{n_{g,\eps}}$ denote those cubes satisfying 
$Q_{g,\eps,i} \cap E_{g,\eps} \neq \emptyset$, so that
\[
   \bigcup_{i=1}^{n_{g,\eps}} Q_{g,\eps,i} \supseteq E_{g,\eps}.
\]
Moreover, since \(M_g>0\), there exists \(\bmx_0\in[0,1]^d\) such that
\(g(\bmx_0)=M_g\). By continuity, \(E_{g,\eps}\) contains a nontrivial relative
neighborhood of \(\bmx_0\) in \([0,1]^d\). Therefore, by choosing
\(\tildedelta_{g,\eps}\in(0,\delta_{g,\eps})\) sufficiently small, we may ensure that
\(E_{g,\eps}\) intersects at least two distinct grid cubes, and hence
\(n_{g,\eps}\ge 2 .\)
After possibly decreasing \(\tildedelta_{g,\eps}\), the
monotonicity of the modulus of continuity also gives
\[
\omega_{g^+}(2\sqrt d\,\widetilde\delta_{g,\varepsilon})
\le (1-\varepsilon-2^d\varepsilon)M_g .
\]

\item Define
\[
   T_\eps g 
   \coloneqq 
   \eps M_g \sum_{i=1}^{n_{g,\eps}} \Gamma_{Q_{g,\eps,i}},
\]
where $\Gamma_{Q_{g,\eps,i}}$ is the cutoff function associated with 
$Q_{g,\eps,i}$, as introduced in Definition~\ref{def:dilated-cube-cutoff}.
\end{enumerate}

We remark that the operator \(T_\varepsilon\) constructed above satisfies the defining properties of a positive \(\varepsilon\)-contraction, as will be established in Lemma~\ref{lem:T:positive:eps:contraction}. 
Before proving this result, we introduce an auxiliary geometric lemma that plays a crucial role in controlling the overlap of the supports arising in the construction of $T_\epsilon$.

\begin{remark}
The dilation factor $r\in(1,2)$ is introduced to balance two competing requirements. On the one hand, it ensures that any point sufficiently close to a cube intersecting $E_{g,\varepsilon}$ lies within a uniformly bounded neighborhood, which is essential for deriving lower bounds on 
$g$ via continuity estimates. On the other hand, restricting $r<2$
guarantees a uniform upper bound on the number of overlapping dilated cubes at any point. This balance is fundamental to establishing both the positivity and the contractivity properties of $T_\epsilon$.
\end{remark}

\begin{lemma}[Bounded overlap under small centered dilation]
\label{lem:bounded:overlap:scaling}
Let $\delta>0$ and consider the axis-aligned grid of cubes
\[
     Q_{\bm{\beta}}
     := \prod_{j=1}^d [\beta_j\delta,\;(\beta_j+1)\delta]\quad \tn{for all }  \bm{\beta}\in\mathbb{Z}^d.
\]
Then, for any $r\in(1,2)$ and every $\bmx\in\mathbb{R}^d$,
\[
  \bigl|\{\bm{\beta}\in\mathbb{Z}^d : \bmx\in rQ_{\bm{\beta}}\}\bigr|
  \le 2^d .
\]
\end{lemma}
\begin{proof}
Each cube can be written in centered form as
\[
  Q_{\bm{\beta}}
  = \bm{c}_{\bm{\beta}} + \Big[-\frac{\delta}{2},\,\frac{\delta}{2}\Big]^d 
  \quad \tn{and}\quad 
  \bm{c}_{\bm{\beta}} := \delta\bigl(\bm{\beta}+\tfrac12\mathbf{1}\bigr)\quad \tn{for all } \bm{\beta}\in\mathbb{Z}^d.
\]
Fix $\bmx\in\mathbb{R}^d$.  
The condition $\bmx\in rQ_{\bm{\beta}}$ is equivalent,  coordinatewise, to
\[
  |x_j - c_{\bm{\beta},j}| \le \frac{r\delta}{2}
  \quad  \tn{for  } j=1,2,\dots,d.
\]
Substituting $c_{\bm{\beta},j}:=\delta(\beta_j+\tfrac12)$, we obtain
\[
  \Bigl|\frac{x_j}{\delta} - \bigl(\beta_j+\tfrac12\bigr)\Bigr|
  \le \frac{r}{2}.
\]
Thus, for each coordinate 
$j$, the index $\beta_j$ must lie in an interval of length $r<2$ on the real line.  
Any interval of length strictly less than $2$ contains at most two integers, and hence there are at most two admissible values of $\beta_j$ for each coordinate.

Since the coordinates are independent, the total number of multi-indices 
$\bm{\beta}:=(\beta_1,\dots,\beta_d)$ satisfying $\bmx\in rQ_{\bm{\beta}}$ is bounded above by $2^d$.
\end{proof}

\begin{remark}
The bound in Lemma~\ref{lem:bounded:overlap:scaling} is sharp. When $\bmx$ is a vertex of the grid, exactly $2^d$ dilated cubes contain $\bmx$ for any $r\in (1,2)$.  The
overlap can exceed $2^d$ only when $r\delta/2 \ge \delta$, that is, when $r\ge 2$, corresponding to a dilation that at least doubles the side length of the cubes.
\end{remark}

With Lemma~\ref{lem:bounded:overlap:scaling} in hand, we are prepared to verify that
the operator \(T_\varepsilon\) satisfies conditions~\ref{item:T:C1}
and~\ref{item:T:C2}. By Lemma~\ref{lem:P1P2:to:positive:eps:constraction},
these two conditions together imply that \(T_\varepsilon\) is a positive \(\varepsilon\)-contraction,
as stated in Lemma~\ref{lem:T:positive:eps:contraction} below.

\begin{lemma}\label{lem:T:positive:eps:contraction}
Let $T_\varepsilon$ be the operator defined above. 
Then $T_\varepsilon$ satisfies conditions~\ref{item:T:C1} and~\ref{item:T:C2}. 
Consequently, $T_\varepsilon$ is a positive $\varepsilon$-contraction.
\end{lemma}

\begin{proof}
We verify conditions~\ref{item:T:C1} and~\ref{item:T:C2}.

\needspace{4\baselineskip}
\medskip
\noindent\textbf{Verification of condition~\ref{item:T:C1}.}

By the construction of $T_\epsilon$,
\[
T_\varepsilon g(\bmx)
= \varepsilon M_g \sum_{i=1}^{n_{g,\varepsilon}} \Gamma_{Q_{g,\varepsilon,i}}(\bmx)
\ge 0 ,
\]
since each cutoff function is nonnegative.  
We may assume $M_g>0$, as the case $M_g=0$ implies $T_\varepsilon g\equiv 0$ and is trivial.

Fix any
\[
\bmx \in  {U}_{g,\epsilon}:=\Bigl( \bigcup_{i=1}^{n_{g,\varepsilon}} rQ_{g,\varepsilon,i} \Bigr)\cap[0,1]^d.
\]
By Lemma~\ref{lem:bounded:overlap:scaling}, the point $\bmx$ belongs to at most $2^d$
dilated cubes. Therefore,
\[
0 \le \sum_{i=1}^{n_{g,\varepsilon}} \Gamma_{Q_{g,\varepsilon,i}}(\bmx) \le 2^d,
\]
which yields
\begin{equation}
\label{eq:Tg:upper}
0 \le T_\varepsilon g(\bmx)
\le 2^d \varepsilon M_g \quad \tn{for all }
\bmx\in {U}_{g,\epsilon}.
\end{equation}

Now choose an index $i$ such that $\bmx\in rQ_{g,\varepsilon,i}$.
Since $Q_{g,\varepsilon,i}\cap E_{g,\varepsilon}\neq\emptyset$ by construction,
there exists $\tilde{\bmx}\in Q_{g,\varepsilon,i}\cap E_{g,\varepsilon}$.
Because both $\bmx$ and $\tilde{\bmx}$ lie in the same dilated cube $rQ_{g,\varepsilon,i}$,
\[
\|\bmx-\tilde{\bmx}\|_2
\le r\sqrt{d}\,\tildedelta_{g,\varepsilon}
\le 2\sqrt{d}\,\tildedelta_{g,\varepsilon}.
\]
Recalling that
\[
E_{g,\varepsilon}=\{\bmx: g(\bmx)> (1-\varepsilon)M_g\}
\]
and using the definition of $\tildedelta_{g,\varepsilon}$, we obtain
\[
\begin{aligned}
g^+(\bmx)
&\ge g^+(\tilde{\bmx}) - \omega_{g^+}(2\sqrt{d}\,\tildedelta_{g,\varepsilon}) \\
&\ge (1-\varepsilon)M_g - (1-\varepsilon-2^d\varepsilon)M_g
= 2^d\varepsilon M_g>0.
\end{aligned}
\]
Thus,
\begin{equation}
\label{eq:g:lower}
2^d\varepsilon M_g \le g^+(\bmx)=g(\bmx) \le M_g\quad \tn{for all }
\bmx\in  {U}_{g,\epsilon}.
\end{equation}

If instead
\[
\bmx\notin \bigcup_{i=1}^{n_{g,\varepsilon}} rQ_{g,\varepsilon,i},
\]
then $\Gamma_{Q_{g,\varepsilon,i}}(\bmx)=0$ for all $i$, and hence
\begin{equation}
\label{eq:Tg:outside}
T_\varepsilon g(\bmx)=0 .
\end{equation}
Combining \eqref{eq:Tg:upper}, \eqref{eq:g:lower}, and \eqref{eq:Tg:outside}, we conclude that
\[
0 \le T_\varepsilon g(\bmx) \le g^+(\bmx)
\quad  \text{for all }\bmx\in[0,1]^d,
\]
which establishes condition~\ref{item:T:C1}.

\needspace{4\baselineskip}
\medskip
\noindent\textbf{Verification of condition~\ref{item:T:C2}.}

Let $\bmx\in[0,1]^d$ satisfy
\(
g(\bmx) >  (1-\varepsilon)M_g.
\)
Then $\bmx\in E_{g,\varepsilon}$ and hence belongs to some cube $Q_{g,\varepsilon,i}$.
By definition of the cutoff function,
\(
\Gamma_{Q_{g,\varepsilon,i}}(\bmx)=1.
\)
Therefore,
\[
T_\varepsilon g(\bmx)
= \varepsilon M_g \sum_{i=1}^{n_{g,\varepsilon}} \Gamma_{Q_{g,\varepsilon,i}}(\bmx)
\ge \varepsilon M_g ,
\]
which verifies condition~\ref{item:T:C2}.

\medskip
Both conditions are satisfied, and hence $T_\varepsilon$ is a positive
$\varepsilon$-contraction.
\end{proof}

\subsection{Main argument}
\label{sec:detailed:proof:thm:main:mgdl:one:op}

We now assemble the components developed above.
The operator $T_\varepsilon$ provides a one-sided contraction on the positive
part of a continuous function, while its symmetrization
$S_\varepsilon$ yields a balanced $\eps$-contraction on the entire function.
We now show that, for each fixed $f$, the correction $S_\varepsilon f$
can be realized by a finite multigrade \ReLU{} architecture, thereby
proving Theorem~\ref{thm:main:mgdl:one:op}.

\begin{proof}[Proof of Theorem~\ref{thm:main:mgdl:one:op}]
The proof proceeds in two stages.

\medskip
\noindent\textbf{Stage I: Network realization of the contraction.}

Let $T_\varepsilon$ be the operator constructed in the previous section. Throughout the proof, we apply $T_\varepsilon$
to both $f$ and $-f$; hence all quantities
\[
M_g,\quad  \Gamma_{Q_{g,\varepsilon,i}},\quad  n_{g,\varepsilon}
\]
are understood with $g=f$ or $g=-f$.

By definition, for $\bmx\in[0,1]^d$,
\[
T_\varepsilon f(\bmx)
= \varepsilon M_f \sum_{i=1}^{n_{f,\varepsilon}}
\Gamma_{Q_{f,\varepsilon,i}}(\bmx),
\]
and likewise,
\[
T_\varepsilon(-f)(\bmx)
= \varepsilon M_{-f} \sum_{i=1}^{n_{-f,\varepsilon}}
\Gamma_{Q_{-f,\varepsilon,i}}(\bmx).
\]
We remark that $n_{f,\varepsilon}$ and $n_{-f,\varepsilon}$ may be zero. For instance, $n_{f,\varepsilon}=0$ if $f(\bm{x}) \le 0$ for all $\bm{x}\in [0,1]^d$. 
By construction, \(n_{g,\varepsilon}\ge2\) whenever \(M_g>0\). Since
\(f\not\equiv0\), at least one of \(M_f\) and \(M_{-f}\) is positive, and hence
\(n_{f,\varepsilon}+n_{-f,\varepsilon}\ge2 .\)


By Lemma~\ref{lem:T:positive:eps:contraction},
$T_\varepsilon$ is a positive $\varepsilon$-contraction.
Lemma~\ref{lem:positive:to:balanced:eps:contraction} then implies
that the symmetrized operator $S_\varepsilon$, defined by
\[
S_\varepsilon g \coloneqq T_\varepsilon g-T_\varepsilon(-g),
\]
is a balanced $\varepsilon$-contraction.

For notational convenience, define
\[
n \coloneqq n_{f,\varepsilon} + n_{-f,\varepsilon} \ge2,
\]
and introduce a unified sequence $\{\hat M_k\}_{k=1}^n$ by
\[
\hat M_k :=
\begin{cases}
M_f  & \tn{for } k = 1,\dots,n_{f,\varepsilon},\\[2pt]
- M_{-f}  & \tn{for } k = n_{f,\varepsilon}+1,\dots,n,
\end{cases}
\]
and
$\{\hat Q_k\}_{k=1}^n$ by
\[
\hat Q_k :=
\begin{cases}
Q_{f,\varepsilon,k}  & \tn{for }  k = 1,\dots,n_{f,\varepsilon},\\[2pt]
Q_{-f,\varepsilon,k-n_{f,\varepsilon}}  & \tn{for }  k = n_{f,\varepsilon}+1,\dots,n.
\end{cases}
\]
Then
\[
S_\varepsilon f
= \sum_{k=1}^n h_k,
\quad  \tn{where}\   
h_k \coloneqq \varepsilon \hat M_k\, \Gamma_{\hat Q_k}.
\]

We remark that, according to the definition of $T_\varepsilon$, the value of
$n$ depends primarily on the modulus of continuity $\omega_f$ (which reflects
the complexity of the target function $f$) and the prescribed $\varepsilon$.

    Let $\ell_k$ and $\bmc_k$ denote the side length and center of $\hat Q_k$.
By Definition~\ref{def:dilated-cube-cutoff},
\[
\Gamma_{\hat Q_k}(\bmx)
= \Psi\!\left(\tfrac{2}{\ell_k}(\bmx-\bmc_k)\right),
\]
and hence there exist affine maps $\caltildeA_{k,1}\in \aff(d,4d)$ and $\caltildeA_{k,2}\in \aff(4d,1)$ such that
\[
\Gamma_{\hat Q_k}=\ReLU\circ  \caltildeA_{k,2}\circ \ReLU\circ \caltildeA_{k,1}.
\]

As illustrated in Figure~\ref{fig:h:k}, there exist affine maps
\begin{gather*}
\mathcal{A}_{1}\in\aff(d,5d+1), \quad
\calA_2,\ldots,\calA_{n}\in\aff(5d+1,5d+1),\quad \calA_{n+1}\in\aff(5d+1,1),
\\[3pt]
\mathcal{A}_{2}^{\mathrm{out}},\ldots, \mathcal{A}_{n}^{\mathrm{out}}\in\aff(5d+1, 1),\quad  \calA_{n+1}^{\mathrm{out}}\in\aff(1, 1)
\end{gather*}
such that
\[
h_k
= \mathcal{A}_{k+1}^{\mathrm{out}}\circ
\ocomp_{i=1}^{k+1} (\ReLU\circ\mathcal{A}_i).
\]

\begin{figure}
\centering
\includegraphics[width=0.9985\textwidth]{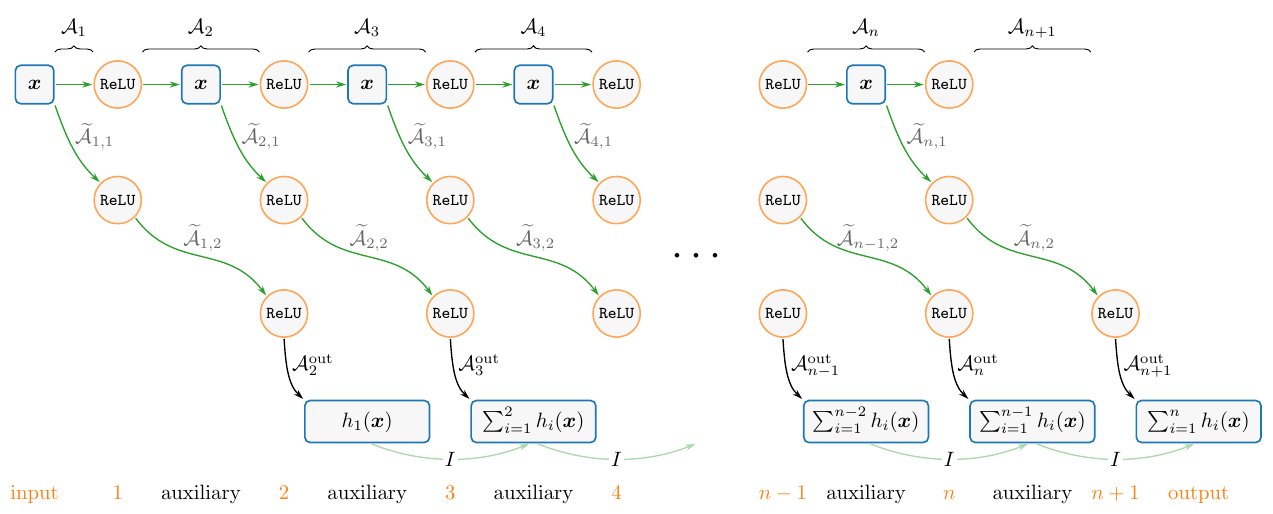}
\caption{An illustration of the network architecture used to realize $h_k$ for $\bmx\in[0,1]^d$.}
\label{fig:h:k}
\end{figure}

More precisely, for any $\bmx\in\mathbb{R}^d$, $\bmy\in\mathbb{R}^{4d}$, and
$z\in\mathbb{R}$, we can define
\begin{gather*}
    \calA_1(\bmx)\coloneqq \begin{bmatrix}
        \bmx\\[1pt]  \tildecalA_{1,1}(\bmx) \\[1pt] 0
    \end{bmatrix}
,\quad
    \calA_{n+1}\left(
    \begin{bmatrix}
        \bmx\\  \bmy \\ z
    \end{bmatrix}
    \right)\coloneqq   \tildecalA_{n,2}(\bmy),\quad 
    \calA_{n+1}^{\mathrm{out}}(z)\coloneqq  \varepsilon \hat M_n  z,
\\[3pt]
        \calA_k\left(
    \begin{bmatrix}
        \bmx\\   \bmy \\   z
    \end{bmatrix}
    \right)\coloneqq  \begin{bmatrix}
        \bmx\\[1pt]  \tildecalA_{k,1}(\bmx) \\[1pt]  \tildecalA_{k-1,2}(\bmy)
    \end{bmatrix},\quad 
    \calA_k^{\mathrm{out}}\left(
    \begin{bmatrix}
        \bmx\\  \bmy \\   z
    \end{bmatrix}
    \right)\coloneqq   \varepsilon \hat M_{k-1} z\quad  \tn{for $k=2,\ldots,n$.}
\end{gather*}
We note that, for $k=2,\ldots,n$, the feature vector produced
after $k$ hidden layers takes the form
\begin{equation*}
    \begin{bmatrix}
        \ReLU(\bmx)\\[2pt]  \ReLU\big(\tildecalA_{k,1}(\bmx)\big) \\[2pt]   \Gamma_{\hatQ_{k-1}}(\bmx)
    \end{bmatrix}
    =\begin{bmatrix}
        \bmx\\[2pt]  \ReLU\big(\tildecalA_{k,1}(\bmx)\big) \\[2pt]   \Gamma_{\hatQ_{k-1}}(\bmx)
    \end{bmatrix}
    \in \R^{5d+1} \quad \tn{for all } \bmx\in [0,1]^d,
\end{equation*}
from which we deduce
\[
\left(
\ocomp_{i=1}^n (\ReLU\circ\mathcal{A}_i)(\bmx)
\right)[1:d]
=
\bmx .
\]
The architecture in Figure~\ref{fig:h:k} realizes the correction $S_\varepsilon f$ as a residual
\ReLU{} network of width $5d+1$ and depth $n+1$, where $I$ denotes the identity
map. 
The black arrow represents the affine output map
$\mathcal{A}_{k+1}^{\mathrm{out}}$, which provides the final scaling to obtain $h_k$.
Each intermediate map $\mathcal{A}_i$ is
formed by composing the affine maps adjacent to an ``auxiliary layer'', which
contains no activation and is introduced only to display the decomposition of
$h_k$. Summing these grade-wise blocks therefore gives exactly the recursive
structure required in Theorem~\ref{thm:main:mgdl:one:op}.


\medskip
\noindent\textbf{Stage II: Decay of the residual.}

Clearly, $R_1 = f$ and
\[
R_{k+1} =  R_{k} - \mathcal{A}_{k+1}^{\mathrm{out}} \circ 
   \ocomp_{i=1}^{k+1} (\ReLU \circ \mathcal{A}_i)
   =R_{k}-h_k
   \quad   \tn{for}\   k=1,\dots,n.
\]
By construction,
\[
R_{n+1} = f - \sum_{k=1}^n h_k = f - S_\varepsilon f.
\]
Since $S_\varepsilon$ is a balanced $\varepsilon$-contraction,
\[
\|R_{n+1}\|_{L^\infty([0,1]^d)}
= \|f - S_\varepsilon f\|_{L^\infty([0,1]^d)}
\le (1-\varepsilon)\,\|f\|_{L^\infty([0,1]^d)}.
\]

We now show that each update produces pointwise monotone decay and a strict
decrease in the $L^p$-norm.
By Lemma~\ref{lem:T:positive:eps:contraction},
\[
0 \le T_\varepsilon f(\bmx) \le f^+(\bmx),
\quad 
0 \le T_\varepsilon(-f)(\bmx) \le f^-(\bmx),
\]
which implies
\begin{equation}\label{eq:sum:hk:ub}
    \sum_{k=1}^{n_{f,\varepsilon}} h_k(\bmx) \le f^+(\bmx),
\quad 
-\sum_{k=n_{f,\varepsilon}+1}^{n} h_k(\bmx) \le f^-(\bmx).
\end{equation}

Let
\[
\Omega_1 \coloneqq \{\bmx\in [0,1]^d: f(\bmx)\ge 0\},
\quad 
\Omega_2 \coloneqq \{\bmx\in [0,1]^d: f(\bmx)<0\}
\]
and observe that $[0,1]^d=\Omega_1\cup\Omega_2$.

\medskip
\noindent\emph{Case 1: $k\le n_{f,\varepsilon}$ (positive-side update).}

We first consider the pointwise estimate.
For $\bmx\in\Omega_1$,
by \eqref{eq:sum:hk:ub},
\[
R_{k+1}(\bmx)
= f(\bmx)-\sum_{j=1}^k h_j(\bmx)
= f^{+}(\bmx)-\sum_{j=1}^k h_j(\bmx)
\ge 0 .
\]
Moreover, we have $h_k(\bmx)=\eps M_f \Gamma_{\hatQ_k}(\bmx) \ge0$ for $\bmx\in\Omega_1$ and hence
\[
0 \le R_{k+1}(\bmx) = R_k(\bmx)-h_k(\bmx)\le R_k(\bmx).
\]
For $\bmx\in\Omega_2$, $h_k(\bmx)=0$, so $R_{k+1}(\bmx)=R_k(\bmx)$.
Thus 
\[
|R_{k+1}(\bmx)| \le |R_k(\bmx)|\quad \tn{for }    \bmx \in \Omega_1 \cup \Omega_2 = [0,1]^d .
\]

We next  consider the $L^p$-norm  estimate for $p\in [1,\infty)$. 
Since \(\hat Q_k\cap [0,1]^d\neq \emptyset\) and \(h_k(\bm{x})>0\) on
\(\hat Q_k\), there exists \(\hat r\in(1,r)\), sufficiently close to \(1\),
such that
\[
    h_k(\bm{x})>0
    \quad \text{on } \hat  r\hat Q_k\cap[0,1]^d,
\]
where \(\hat r\hat Q_k\cap[0,1]^d\) has positive measure. By the construction
of the cutoff functions, we have
\[
    \hat r\hat Q_k\cap[0,1]^d \subseteq \Omega_1 .
\]
Moreover, \(0\le h_k(\bm{x})\le R_k(\bm{x})\) on \(\Omega_1\), while   \(h_k(\bm{x})=0\) on
\(\Omega_2\). Hence
\[
\begin{aligned}
    \int_{\Omega_1} |R_{k+1}(\bm{x})|^p \, d\bm{x}
    = \int_{\Omega_1} |R_k(\bm{x})-h_k(\bm{x})|^p \, d\bm{x}  < \int_{\Omega_1} |R_k(\bm{x})|^p \, d\bm{x},
\end{aligned}
\]
where the strict inequality follows from the fact that \(h_k>0\) on a subset
of \(\Omega_1\) with positive measure. On the other hand,
\[
\begin{aligned}
    \int_{\Omega_2} |R_{k+1}(\bm{x})|^p \, d\bm{x}
    = \int_{\Omega_2} |R_k(\bm{x})-h_k(\bm{x})|^p \, d\bm{x} = \int_{\Omega_2} |R_k(\bm{x})|^p \, d\bm{x}.
\end{aligned}
\]
Therefore,
\[
    \|R_{k+1}\|_{L^p([0,1]^d)}
    <
    \|R_k\|_{L^p([0,1]^d)} .
\]

\Needspace{4\baselineskip}
\medskip
\noindent\emph{Case 2: $k>n_{f,\varepsilon}$ (negative-side update).}

We first consider the pointwise estimate.
For $\bmx \in \Omega_2$, we have
 $f^{+}(\bmx)=0$, implying
\[
h_j(\bmx)=0 \quad  \tn{for all }    j \le n_{f,\eps}.
\]
Then by \eqref{eq:sum:hk:ub},
\[
R_{k+1}(\bmx)
= f(\bmx)-\sum_{j=1}^k h_j(\bmx)
= -f^{-}(\bmx)-\sum_{j=n_{f,\eps}+1}^k h_j(\bmx)
\le 0 .
\]
Moreover,  \[
h_k(\bmx) = -\eps M_{-f}\,\Gamma_{\hat Q_k}(\bmx) \le 0 \quad  \tn{for all }   
\bmx\in\Omega_2,
\]
from which we deduce
\[
0 \ge R_{k+1}(\bmx)
= R_k(\bmx)-h_k(\bmx)
\ge R_k(\bmx) \quad  \tn{for all }    \bmx\in\Omega_2 .
\]
For $\bmx \in \Omega_1$, we have $h_k(\bmx)=0$, and hence
\[
R_{k+1}(\bmx)
= R_k(\bmx) - h_k(\bmx)
= R_k(\bmx).
\]
Consequently,
\[
|R_{k+1}(\bmx)| \le |R_k(\bmx)| \quad  \tn{for all }     \bmx \in \Omega_1 \cup \Omega_2 = [0,1]^d .
\]

We next  consider the $L^p$-norm  estimate for $p\in [1,\infty)$.
Since \(\hat Q_k\cap [0,1]^d\neq \emptyset\) and \(h_k(\bm{x})<0\) on
\(\hat Q_k\), there exists \(\tilde r\in(1,r)\), sufficiently close to \(1\),
such that
\[
    h_k(\bm{x})<0
    \quad \text{on } \tilde r\hat Q_k\cap[0,1]^d,
\]
where \(\tilde r\hat Q_k\cap[0,1]^d\) has positive measure. By the construction
of the cutoff functions, we have
\[
    \tilde r\hat Q_k\cap[0,1]^d \subseteq \Omega_2 .
\]
Moreover,  \(R_k(\bm{x})\le h_k(\bm{x})\le 0\) on \(\Omega_2\), while  \(h_k(\bm{x})=0\) on
\(\Omega_1\). Hence
\[
\begin{aligned}
    \int_{\Omega_2} |R_{k+1}(\bm{x})|^p \, d\bm{x}
    &= \int_{\Omega_2} |R_k(\bm{x})-h_k(\bm{x})|^p \, d\bm{x}  < \int_{\Omega_2} |R_k(\bm{x})|^p \, d\bm{x},
\end{aligned}
\]
where the strict inequality follows from the fact that \(h_k<0\) on a subset
of \(\Omega_2\) with positive measure. On the other hand,
\[
\begin{aligned}
    \int_{\Omega_1} |R_{k+1}(\bm{x})|^p \, d\bm{x}
    &= \int_{\Omega_1} |R_k(\bm{x})-h_k(\bm{x})|^p \, d\bm{x} = \int_{\Omega_1} |R_k(\bm{x})|^p \, d\bm{x}.
\end{aligned}
\]
Therefore,
\[
    \|R_{k+1}\|_{L^p([0,1]^d)}
    <
    \|R_k\|_{L^p([0,1]^d)} .
\]

Combining both cases, for $k=1,2,\dots,n$, we have
\[
|R_{k+1}(\bmx)| \le |R_k(\bmx)| \quad \tn{for all }  \bmx\in[0,1]^d,
\]
and for every $p\in [1,\infty)$,
\[
\|R_{k+1}\|_{L^p([0,1]^d)} < \|R_k\|_{L^p([0,1]^d)},
\]
which completes the proof.
\end{proof}

\section{Conclusion}\label{sec:conclusion}

In this work, we establish a rigorous theoretical foundation for MGDL as a structured error refinement framework within function approximation theory.
By formulating the grade-wise training process in operator-theoretic terms, we identify explicit structural conditions under which a fixed-width multigrade \ReLU{} architecture 
    achieves pointwise residual domination, strict \(L^p\) residual decrease at
each nontrivial grade for $p\in [1,\infty)$, and uniform convergence to the target function.
These results provide a constructive and transparent explanation of how approximation accuracy can be progressively enhanced with depth through controlled, hierarchical residual updates.

Our analysis positions MGDL as a principled alternative to conventional end-to-end training.
Rather than optimizing all layers simultaneously, MGDL addresses
approximation error through successive residual refinements.
Our construction shows that these refinements can be made
monotone, clarifying the functional role of depth in deep
neural networks.
The resulting theoretical guarantees not only complement existing empirical observations but also connect MGDL to classical concepts in approximation theory and hierarchical modeling.

Several directions for future research naturally follow from this work.
A primary extension is to broaden the analysis beyond fixed-width \ReLU{} networks to encompass other activation functions and architectural families.
Another open problem concerns the optimization landscape of practical MGDL implementations, in which network parameters are learned from data rather than constructed analytically, and how such training dynamics affect stability and generalization.
Finally, integrating multigrade principles into modern architectures, including transformers, neural operators, and physics-informed models, offers a promising pathway toward scalable learning systems that unify theoretical guarantees with practical effectiveness.

\section*{Acknowledgments}

Shijun Zhang was partially supported by
the start-up fund  P0053092  from 
 The 
Hong Kong Polytechnic University.
Zuowei Shen was partially supported under the Distinguished Professorship of the National University of Singapore.
Yuesheng Xu was supported in part by the US National Science Foundation under Grant
DMS-2208386.


\hypersetup{
citecolor=black,linkcolor=black,urlcolor=black}
\renewcommand{\doi}[1]{\textnormal{\doitext}~\texttt{\href{https://doi.org/#1}{\detokenize{#1}}}}

\bibliographystyle{plainnat}   
\bibliography{references}
\end{document}

%% file: preamble.tex
\usepackage[a4paper]{geometry}		
\usepackage{needspace}
\usepackage{amsthm}
\usepackage{amsmath,scalerel}
\usepackage[numbers,sort]{natbib} 

\usepackage{amsmath,amsfonts,amssymb}
\usepackage[scr=rsfs]{mathalpha} 
\usepackage{mathtools}
\usepackage{dsfont} 
\usepackage{bm}

\usepackage{xurl,xcolor
}

\usepackage{enumitem}

\setlist[enumerate]{
  topsep=3pt,
  itemsep=3pt,
  parsep=0pt,
  partopsep=0pt,
  leftmargin=*,   %
  align=right,
  labelsep=0.600em,
  itemindent=0.00em , %
  labelindent=0.6em,
}

\setlist[itemize]{
  topsep=3pt,
  itemsep=3pt,
  parsep=0pt,
  partopsep=0pt,
  leftmargin=*,   %
  align=right,
  labelsep=0.600em,
  itemindent=0.00em , %
  labelindent=0.6em,
}

\usepackage{tabularx,multirow,array,booktabs}
\usepackage{colortbl}
\usepackage{graphicx}
\graphicspath{{figures/}}%
\usepackage{float}
\usepackage{subcaption} 

\usepackage{caption}
\numberwithin{equation}{section}
\usepackage{amsthm}
\theoremstyle{plain}
\newtheorem{theorem}{Theorem}[section]%
\newtheorem{corollary}[theorem]{Corollary} %
\newtheorem{lemma}[theorem]{Lemma}
\newtheorem*{lemma*}{Lemma}

\theoremstyle{definition}
\newtheorem{definition}[theorem]{Definition}
\theoremstyle{remark}

\newtheorem{remark}[theorem]{Remark}
\newtheorem*{remark*}{Remark}
\usepackage{etoolbox,xparse,dsfont,bm,amssymb}
\newcommand{\myMFabc}[4]{\expandafter#1\csname#3#4\endcsname{{#2{#4}}}}
\newcommand{\myMFcmd}[4]{\expandafter#1\csname#3#4\endcsname{{#2{\csname#4\endcsname}}}}

\newcommand{\MFabc}[3][\newcommand]{
    \def\doOld##1##2{\forcsvlist{\myMFabc{#1}{##1}{##2}}{#3}}
    \providecommand{\do}{do}
    \RenewDocumentCommand \do { >{\SplitList{,}} m } { \doOld##1 }
    \docsvlist{#2}
}
\newcommand{\MFcmd}[3][\newcommand]{
    \def\doOld##1##2{\forcsvlist{\myMFcmd{#1}{##1}{##2}}{#3}}
    \providecommand{\do}{do}
    \RenewDocumentCommand \do { >{\SplitList{,}} m } { \doOld##1 }
    \docsvlist{#2}
}

\MFabc{ {\mathfrak,frak}, }{T,u,U,o,O,E,m}
\MFabc{ {\mathbb, }, }{A,N,Z,R,C,Q}

\newcommand{\hatbm}[1]{\widehat{\bm{#1}}}
\newcommand{\tildebm}[1]{\widetilde{\bm{#1}}}
\newcommand{\bmcal}[1]{\bm{\mathcal{#1}}}
\newcommand{\caltilde}[1]{\mathcal{\widetilde{#1}}}
\newcommand{\calhat}[1]{\mathcal{\widehat{#1}}} 
\newcommand{\scrtilde}[1]{\widetilde{\mathscr{#1}\mspace{1mu}\mspace{-1mu}}}
\newcommand{\scrhat}[1]{\mathscr{\widehat{#1}\mspace{1mu}\mspace{-1mu}}}
\newcommand{\bmcalhat}[1]{\bm{\mathcal{\widehat{#1}}}}
\newcommand{\bmcaltilde}[1]{\bm{\mathcal{\widetilde{#1}}}}
\MFabc{ {\mathcal,cal}, {\mathscr,scr}, {\mathbb,bb}, {\bmcal,bmcal}, {\bmcal,calbm}, {\caltilde,caltilde}, {\caltilde,tildecal}, {\calhat,calhat}, {\calhat,hatcal}, {\scrtilde,scrtilde}, {\scrtilde,tildescr}, {\scrhat,scrhat}, {\scrhat,hatscr}, {\bmcalhat,bmcalhat}, {\bmcalhat,bmhatcal}, 
{\bmcalhat,calbmhat}, {\bmcalhat,calhatbm}, {\bmcalhat,hatbmcal}, {\bmcalhat,hatcalbm}, {\bmcaltilde,bmcaltilde}, {\bmcaltilde,bmtildecal}, {\bmcaltilde,calbmtilde}, {\bmcaltilde,caltildebm}, {\bmcaltilde,tildebmcal}, {\bmcaltilde,tildecalbm}}{A,B,C,D,E,F,G,H,I,J,K,L,M,N,O,P,Q,R,S,T,U,V,W,X,Y,Z}

\MFabc{{\bm,bm}, {\mathsf,sf}, {\widehat,hat}, {\widetilde,tilde}, {\hatbm,hatbm}, {\hatbm,bmhat}, {\tildebm,tildebm}, {\tildebm,bmtilde}}{a,b,c,d,e,f,g,h,i,j,k,l,m,n,o,p,q,r,s,t,u,v,w,x,y,z,A,B,C,D,E,F,G,H,I,J,K,L,M,N,O,P,Q,R,S,T,U,V,W,X,Y,Z}

\let\eps\varepsilon
\MFcmd{{\bm,bm}, {\widehat,hat}, {\widetilde,tilde}, {\tildebm, tildebm}, {\tildebm, bmtilde}, {\hatbm, hatbm}, {\hatbm, bmhat}}{alpha,beta,gamma,delta,epsilon,zeta,eta,theta,iota,kappa,lambda,mu,nu,xi,omicron,pi,rho,sigma,tau,upsilon,phi,chi,psi,omega,Alpha,Beta,Gamma,Delta,Epsilon,Zeta,Eta,Theta,Iota,Kappa,Lambda,Mu,Nu,Xi,Omicron,Pi,Rho,Sigma,Tau,Upsilon,Phi,Chi,Psi,Omega,varrho,varphi,vartheta,varepsilon,varsigma,ell,eps}

\newcommand{\actdef}[1]{\expandafter\def\csname#1\endcsname{{\ensuremath{\mathtt{#1}}}}}
\forcsvlist{\actdef}{ReLU, LReLU, LeakyReLU, ELU, GELU, SiLU, Softplus, dGELU, dSiLU, dSoftplus, Tanh, Sigmoid, Arctan, Softsign, SRS, dSRS, Swish, dSwish, Mish, dMish, SELU, CELU, dSELU, Sin,SinLU, SinTU, PSinTU, sine, cosine, Sine, Cosine, EUAF}

\newlength{\myLength}

\newenvironment{keywords}{\par \noindent\textbf{Keywords}:}{\par}
\newenvironment{MSCcodes}{\par\par \noindent\textbf{MSC codes}:}{\par}

\DeclareMathOperator*{\argmin}{arg\,min}

\usepackage[left,mathlines]{lineno}
\usepackage{refcount}

\definecolor{mylinenumbercolor}{HTML}{BEBEBE}

\linenumberdisplaymath

\newcommand{\mylineref}[1]{%
  \hyperref[#1zsj]{\textcolor{blue}{\ref{#1}}}%
}



\let\tilde\widetilde
\let\hat\widehat
\let\epsilon\varepsilon
\let\eps\varepsilon
\let\subset\subseteq
\let\tn\textnormal

\let\cdots\customcdots
\let\myforall\forall
\def\forall{{\myforall\, }}
\let\myexists\exists
\def\exists{{\myexists\, }}
\let\emptyset\varnothing

\definecolor{mygray}{RGB}{230,230,230}
\definecolor{myorange}{HTML}{ff7f0e}

\let\cite\citep
\usepackage{doi}
\def\doitext{DOI: }
\makeatletter
\long\def\@makefntext#1{\@setpar{\@@par\@tempdima \hsize 
		\advance\@tempdima-15pt\parshape \@ne 15pt \@tempdima}\par
	\parindent 2em\noindent \hbox to \z@{\hss{\textsuperscript{\@thefnmark}} \hfil}#1}
\newlength\aftertitskip     \newlength\beforetitskip
\newlength\interauthorskip  \newlength\aftermaketitskip
\def\maketitle{\par
	\begingroup
	\def\thefootnote{\color{black}\fnsymbol{footnote}}
	\def\@makefnmark{\hbox to 0pt{$^{\@thefnmark}$\hss}}
	\@maketitle \@thanks
	\endgroup
	\setcounter{footnote}{0}
	\let\maketitle\relax \let\@maketitle\relax
	\gdef\@thanks{}\gdef\@author{}\gdef\@title{}\let\thanks\relax}

\def\@startauthor{\noindent \normalsize\bf}
\def\@endauthor{}
\def\@starteditor{\noindent \small {\bf Editor:~}}
\def\@endeditor{\normalsize}
\def\@maketitle{\vbox{\hsize\textwidth
		\linewidth\hsize \vskip \beforetitskip
		{\begin{center} \Large\bf \@title \par \end{center}} \vskip \aftertitskip
		{\def\and{\unskip\enspace{\rm and}\enspace}%
			\def\addr{\small\it}%
            \def\email{\hfill\small\ttfamily}%
			\def\name{\normalsize\bf}%
			\def\AND{\@endauthor\rm\hss \vskip \interauthorskip \@startauthor}
			\@startauthor \@author \@endauthor}
		\vskip \aftermaketitskip
}}
\makeatother
\usepackage{scalerel,amssymb}
\newcommand{\ocomp}{\mathop{\scalerel*{\circledcirc}{\sum}}}

\def\aff{{\mathsf{Aff}}}

\usepackage{needspace}